\documentclass{amsart}

\makeatletter

\usepackage{amsmath,amssymb,amsthm,mathtools}
\usepackage{booktabs,array,multirow}
\usepackage{graphicx}
\usepackage{xcolor}
\usepackage{enumitem}
\usepackage{microtype}
\usepackage{url}

\@ifundefined{algorithm}{\usepackage{algorithm}}{}
\@ifundefined{State}{\usepackage{algpseudocode}}{}

\@ifpackageloaded{hyperref}{}{%
  \usepackage[colorlinks=true,
              linkcolor=blue!45!black,
              citecolor=blue!45!black,
              urlcolor=blue!45!black]{hyperref}}

\usepackage[initials,backrefs]{amsrefs}

\@ifpackageloaded{cleveref}{}{%
  \usepackage[capitalise,nameinlink]{cleveref}}

\newcommand\project@newtheorem[3]{
  \@ifundefined{#1}{\newtheorem{#1}[#2]{#3}}{}}

\@ifundefined{theorem}{}{}
\project@newtheorem{lemma}{theorem}{Lemma}
\project@newtheorem{proposition}{theorem}{Proposition}
\project@newtheorem{corollary}{theorem}{Corollary}
\project@newtheorem{conjecture}{theorem}{Conjecture}
\theoremstyle{definition}
\project@newtheorem{definition}{theorem}{Definition}
\project@newtheorem{example}{theorem}{Example}
\project@newtheorem{assumption}{theorem}{Assumption}
\project@newtheorem{hypothesis}{theorem}{Hypothesis}
\project@newtheorem{problem}{theorem}{Problem}
\project@newtheorem{question}{theorem}{Question}
\project@newtheorem{notation}{theorem}{Notation}
\theoremstyle{remark}
\project@newtheorem{remark}{theorem}{Remark}
\project@newtheorem{observation}{theorem}{Observation}
\theoremstyle{plain}

\project@newtheorem{thm}{theorem}{Theorem}
\project@newtheorem{lem}{theorem}{Lemma}
\project@newtheorem{prop}{theorem}{Proposition}
\project@newtheorem{cor}{theorem}{Corollary}
\project@newtheorem{conj}{theorem}{Conjecture}
\theoremstyle{definition}
\project@newtheorem{defn}{theorem}{Definition}
\project@newtheorem{exa}{theorem}{Example}
\project@newtheorem{assum}{theorem}{Assumption}
\project@newtheorem{hyp}{theorem}{Hypothesis}
\theoremstyle{remark}
\project@newtheorem{rem}{theorem}{Remark}
\theoremstyle{plain}

\crefname{theorem}{Thm.}{Thms.}                \Crefname{theorem}{Thm.}{Thms.}
\crefname{thm}{Thm.}{Thms.}                    \Crefname{thm}{Thm.}{Thms.}
\crefname{lemma}{Lem.}{Lems.}                  \Crefname{lemma}{Lem.}{Lems.}
\crefname{lem}{Lem.}{Lems.}                    \Crefname{lem}{Lem.}{Lems.}
\crefname{proposition}{Prop.}{Props.}          \Crefname{proposition}{Prop.}{Props.}
\crefname{prop}{Prop.}{Props.}                 \Crefname{prop}{Prop.}{Props.}
\crefname{corollary}{Cor.}{Cors.}              \Crefname{corollary}{Cor.}{Cors.}
\crefname{cor}{Cor.}{Cors.}                    \Crefname{cor}{Cor.}{Cors.}
\crefname{conjecture}{Conj.}{Conjs.}           \Crefname{conjecture}{Conj.}{Conjs.}
\crefname{conj}{Conj.}{Conjs.}                 \Crefname{conj}{Conj.}{Conjs.}
\crefname{definition}{Def.}{Defs.}             \Crefname{definition}{Def.}{Defs.}
\crefname{defn}{Def.}{Defs.}                   \Crefname{defn}{Def.}{Defs.}
\crefname{example}{Ex.}{Exs.}                  \Crefname{example}{Ex.}{Exs.}
\crefname{exa}{Ex.}{Exs.}                      \Crefname{exa}{Ex.}{Exs.}
\crefname{assumption}{Assum.}{Assums.}         \Crefname{assumption}{Assum.}{Assums.}
\crefname{assum}{Assum.}{Assums.}              \Crefname{assum}{Assum.}{Assums.}
\crefname{hypothesis}{Hyp.}{Hyps.}             \Crefname{hypothesis}{Hyp.}{Hyps.}
\crefname{hyp}{Hyp.}{Hyps.}                    \Crefname{hyp}{Hyp.}{Hyps.}
\crefname{problem}{Prob.}{Probs.}              \Crefname{problem}{Prob.}{Probs.}
\crefname{question}{Ques.}{Ques.}              \Crefname{question}{Ques.}{Ques.}
\crefname{notation}{Not.}{Nots.}               \Crefname{notation}{Not.}{Nots.}
\crefname{remark}{Rem.}{Rems.}                 \Crefname{remark}{Rem.}{Rems.}
\crefname{rem}{Rem.}{Rems.}                    \Crefname{rem}{Rem.}{Rems.}
\crefname{observation}{Obs.}{Obs.}             \Crefname{observation}{Obs.}{Obs.}
\crefname{algorithm}{Alg.}{Algs.}              \Crefname{algorithm}{Alg.}{Algs.}
\crefname{equation}{Eq.}{Eqs.}                 \Crefname{equation}{Eq.}{Eqs.}
\crefname{section}{Sec.}{Secs.}                \Crefname{section}{Sec.}{Secs.}
\crefname{subsection}{Sec.}{Secs.}             \Crefname{subsection}{Sec.}{Secs.}
\crefname{figure}{Fig.}{Figs.}                 \Crefname{figure}{Fig.}{Figs.}
\crefname{table}{Tab.}{Tabs.}                  \Crefname{table}{Tab.}{Tabs.}
\crefname{appendix}{App.}{Apps.}               \Crefname{appendix}{App.}{Apps.}

\creflabelformat{equation}{#2\textup{(#1)}#3}

\newcommand\ProvideMathOperator[2]{%
  \@ifundefined{\expandafter\@gobble\string#1}{\DeclareMathOperator{#1}{#2}}{}}

\providecommand{\R}{\mathbb{R}}
\providecommand{\C}{\mathbb{C}}
\providecommand{\Q}{\mathbb{Q}}
\providecommand{\Z}{\mathbb{Z}}

\providecommand{\SO}{\mathrm{SO}}
\providecommand{\Or}{\mathrm{O}}

\providecommand{\SL}{\mathrm{SL}}

\providecommand{\GL}{\mathrm{GL}}

\providecommand{\Sph}{S^{2}}                 
\providecommand{\Harm}{\mathcal{H}}          

\providecommand{\cA}{\mathcal{A}}

\providecommand{\abs}[1]{\lvert #1\rvert}
\providecommand{\norm}[1]{\lVert #1\rVert}
\providecommand{\ip}[2]{\langle #1,#2\rangle}

\ProvideMathOperator{\tr}{tr}
\ProvideMathOperator{\rank}{rank}
\ProvideMathOperator{\diag}{diag}
\ProvideMathOperator{\dist}{dist}
\ProvideMathOperator{\disc}{disc}
\ProvideMathOperator{\Disc}{Disc}
\ProvideMathOperator{\Res}{Res}
\ProvideMathOperator{\lead}{lead}
\ProvideMathOperator{\Fix}{Fix}
\ProvideMathOperator{\Stab}{Stab}
\ProvideMathOperator{\Aut}{Aut}
\ProvideMathOperator{\End}{End}
\ProvideMathOperator{\Hom}{Hom}
\ProvideMathOperator{\Sym}{Sym}
\ProvideMathOperator{\Spec}{Spec}
\ProvideMathOperator{\Jac}{Jac}
\ProvideMathOperator{\Kum}{Kum}
\ProvideMathOperator{\Hess}{Hess}
\ProvideMathOperator{\Gr}{Gr}
\ProvideMathOperator{\Sing}{Sing}
\ProvideMathOperator{\supp}{supp}
\ProvideMathOperator{\sgn}{sgn}
\ProvideMathOperator{\codim}{codim}
\ProvideMathOperator{\Span}{Span}
\ProvideMathOperator{\Id}{Id}
\ProvideMathOperator{\AUC}{AUC}
\ProvideMathOperator{\Emb}{Emb}
\ProvideMathOperator{\Diff}{Diff}
\ProvideMathOperator{\Gram}{Gram}

\makeatother

\usepackage[margin=1.4in]{geometry}
\usepackage{placeins}

\begin{document}

\title{Decoy and disclosure radii of invariant shape descriptors}

\author[T. Shaska]{T. Shaska}
\address{Department of Computer Science and Engineering, Oakland University, Rochester, MI 48309}

\author[L. Beshaj]{L. Beshaj}
\address{Department of Mathematical Sciences, United States Military Academy, West Point, NY 10996}

\subjclass[2020]{Primary 13A50, 68M25; Secondary 14P10, 68T45, 68U05}

\keywords{invariant theory, spherical harmonics, bispectrum, separating invariants, orbit spaces, shape recognition, decoys, disclosure, task certification}

\begin{abstract}
A recognizer that compares rotation-invariant descriptors sees a surface only up to the fiber of the descriptor. We measure this fiber by its radius in the orbit distance from the enrolled surface. A large radius admits decoys, that is, distant shapes that pass the matcher. A small radius discloses the enrolled shape to anyone who captures the stored value. For star-shaped surfaces truncated to spherical harmonics of degree at most $L$, with $n$ coefficients, a descriptor of generic rank $r$ has generic fibers of dimension $n-3-r$ modulo rotations. The standard pool of band powers, even bispectra, and three invariants of the degree-three band therefore admits decoy families of dimension $5$, $13$, $20$ at $L=4,6,8$. Its rank first reaches $n-3$ at $L=16$, and a mirror decoy remains at every $L$. The odd bispectra remove the mirror decoy generically for $L \geq 4$. Yet at fixed mean radius the same pool determines the enclosed volume exactly, and it does not determine whether a surface meets a clearance requirement. We certify two cases by exact and interval arithmetic. At $L=6$ a decoy matches all $32$ invariants to relative precision $2 \cdot 10^{-18}$ at orbit distance at least $0.87$ times the norm of the enrolled tuple. For the radar shape model of asteroid (101955) Bennu, the pool recovers the modeled volume, misses the handedness, and leaves the keep-out radius uncertain by more than $7 \, \mathrm{m}$.
\end{abstract}

\maketitle

\section{Introduction}\label{sec-intro}

Consider a recognizer that identifies three-dimensional objects from their geometry and must ignore the pose of the object. A standard way to do this is to replace the surface by a descriptor that is invariant under rotations, and to compare the descriptor with an enrolled library. The recognizer then sees only the descriptor. Two surfaces with the same descriptor are indistinguishable from that representation alone. The set of such surfaces is the fiber of the descriptor, and we measure it modulo rotations. We analyze this fiber from two adversarial perspectives: decoy construction and disclosure from a released descriptor.

The first adversary builds decoys. A decoy is a surface that is not a rotated copy of an enrolled object but has the same descriptor, up to the tolerance of the recognizer. At zero tolerance a decoy lies in the exact fiber; at positive tolerance it lies in a thickened fiber. Large fibers therefore permit geometrically distinct accepted shapes. A fiber of positive dimension modulo rotations is a whole family of decoys; this ambiguity is intrinsic to the representation and exists independently of classifier training. Every classifier that acts on the descriptor is constant on its fibers, so no training removes them. The second adversary obtains the stored descriptor of an enrolled object. A captured descriptor is consistent with every surface in its fiber, so small fibers disclose the enrolled geometry. We measure both threats by one number, the radius of the fiber in the orbit distance, centered at the enrolled surface. For a fixed admissible set and tolerance, this radius, centered at the enrolled point, has a decoy interpretation and a disclosure interpretation. An observer who knows only the released value faces a different uncertainty, which is governed by the orbit diameter, or by the corresponding minimax radius, of the set of surfaces consistent with that value. The diameter lies between the radius and twice the radius (\cref{prop-minimax-shape}). This statement concerns the shape as a whole. It does not address the accuracy of a particular task, and it is not a statistical notion of privacy.

The setting is that of \cite{sh-181,sh-182}. A star-shaped surface is given by its radial function on the sphere, truncated to spherical harmonics of degree at most $L$. Rotations act band by band on the coefficients, and descriptors are polynomial invariants of the coefficient tuple. A descriptor that separates orbits admits no exact decoys; conversely, an exact released value determines the enrolled orbit. On a compact semialgebraic admissible set, the radius of any descriptor at tolerance $\varepsilon$ exceeds its exact radius by at most $C \varepsilon^{\alpha}$, with constants that may depend on the enrolled tuple. The stability inequality of \cite{sh-181} bounds the orbit distance by a H\"older power of the descriptor discrepancy. It was proved there to certify the descriptor. Read as a disclosure statement, it says that a released descriptor known to accuracy $\varepsilon$ locates the enrolled surface within orbit distance $C \varepsilon^{\alpha}$. On the dipole and quadrupole bands the separation proof of \cite{sh-181} is constructive and yields an explicit reconstruction.

A descriptor that does not separate admits decoys, and the constant-rank theorem measures how many. Let the descriptor have generic rank $r$ on a coefficient space of dimension $n$. Modulo rotations, the part of its fiber in the free maximal-rank locus is a manifold of dimension $n - 3 - r$. If $r = n - 3$, that regular part is a finite union of orbits. If $r < n - 3$, the recognizer accepts a family of decoys of positive dimension near every generic enrolled surface. A descriptor with fewer than $n - 3$ components is always in the second case, whatever its components are.

The standard spectral pool consists of the band powers of \cite{KazhdanFunkhouserRusinkiewicz03}, the even bispectra of \cite{Kakarala12}, and three invariants of the band-three sextic. It is the invariant block used in \cite{sh-182}. We compute its generic rank exactly. For $2 \leq L \leq 15$ the rank equals the size of the pool, which is less than $n - 3$. At the cutoffs $L = 4, 6, 8$ the decoy families have dimension $5$, $13$, and $20$. At $L = 16$ the rank reaches $n - 3$ for the first time, and the continuous ambiguity disappears on the free maximal-rank locus. A discrete ambiguity remains. The pool consists of true scalars, so a surface and its mirror image have the same pool. For every surface that is not a rotated copy of its mirror image, the mirror image is an exact decoy, at every $L$, and at $L = 16$ the finite ambiguity still contains the mirror pair. The odd bispectra are pseudoscalars. For $L\geq4$, adjoining them removes the mirror decoy for generic surfaces; the augmented pool first saturates at $L = 10$. The latter threshold also follows from the degree-three orbit-recovery theorem of \cite{FanEtAl24}.

A certified $L=6$ benchmark complements the dimension count: an integer tuple and a rational tuple agree in all $32$ invariants to relative precision $2 \cdot 10^{-18}$ and have orbit distance at least $0.87$ times the enrolled tuple norm. The rational comparison is exact, and the distance bounds are established by interval arithmetic through a validated covering criterion given in \cref{sec-decoy}.

On the quadrupole band, the decoy radius of the power can be computed exactly. The fiber of the quadrupole power $\norm{Q}^2$ is a sphere. The Hoffman--Wielandt inequality of \cite{HoffmanWielandt53} turns the orbit distance on it into the distance between sorted spectra. The decoy radius at $Q$ is $2 r \sin ( \max \{ \theta, \pi/3 - \theta \} / 2 )$, where $r = \norm{Q}$ and $\cos 3 \theta = 3 \sqrt{6} \det Q / r^3$. It lies between $\frac{\sqrt{6} - \sqrt{2}}{2} \, r$ and $r$, and adjoining the cubic invariant $\det Q$ makes it zero. With the power known exactly, the radius of $( \norm{Q}^2, \det Q )$ at determinant tolerance $\varepsilon$ grows linearly at matrices with distinct eigenvalues and like $\varepsilon^{1/2}$ at matrices with a repeated eigenvalue. On the quadrupole band, the passage from the power spectrum to the bispectrum is the passage from decoys to disclosure.

Whole-shape ambiguity does not settle task utility. We therefore introduce task intervals and prove an exact attribute-disclosure result: at a fixed constant radial term, the even pool determines enclosed volume. An explicit pair of positive radial surfaces has identical descriptors and volumes but opposite outcomes for a specified clearance requirement. The example can be checked by hand, so the conflict is exhibited for a stated task and is not merely inferred from a difference in shape. In \cref{sec-bennu} the same analysis is applied to the radar shape model of the near-Earth asteroid (101955) Bennu, with certified descriptor, distance, positivity, and task bounds for the rounded $L = 6$ model.

The coefficient model assumes a fully observed closed surface that is star-shaped about a specified origin. A centroid convention is an additional constraint. It follows neither from positivity nor from omitting the constant harmonic, and \cref{sec-admissibility} states it explicitly. Surfaces of higher genus and partial views are outside this model. The discussion of applications in \cref{sec-evaluation} proposes a validation setting; it makes no claim about any deployed system.

\section{Descriptors, fibers, and radii}\label{sec-radii}

Let $\R^3$ carry the Euclidean inner product $\ip{x}{y} = x_1 y_1 + x_2 y_2 + x_3 y_3$ and the norm $\abs{x} = \ip{x}{x}^{1/2}$. Let $\Sph = \{ x \in \R^3 : \abs{x} = 1 \}$ be the unit sphere, and let the special orthogonal group $\SO(3)$ be the group of real $3 \times 3$ matrices $R$ with $R^{\top} R = I$ and $\det R = 1$, where $I$ is the identity matrix. Every $R \in \SO(3)$ maps $\Sph$ onto $\Sph$.

Let $l \geq 1$ be an integer. A polynomial $P \in \R [ x_1, x_2, x_3 ]$ is harmonic if $\Delta P = 0$, where $\Delta = \partial_1^2 + \partial_2^2 + \partial_3^2$. Let $\Harm_l$ be the real vector space of the functions $\Sph \to \R$ that are restrictions of homogeneous harmonic polynomials of degree $l$. Its elements are the real spherical harmonics of degree $l$. The space $\Harm_l$ has dimension $2l + 1$; see \cite[Sec.~2.1]{sh-181} for further details.

Let $R \in \SO(3)$ and $f \in \Harm_l$. Define $R \cdot f \colon \Sph \to \R$ by
\[
( R \cdot f ) (p) = f ( R^{-1} p ) , \qquad p \in \Sph ,
\]
where $R^{-1} p$ is the product of the matrix $R^{-1} = R^{\top}$ with the vector $p$. Let $P$ be a homogeneous harmonic polynomial of degree $l$ whose restriction to $\Sph$ is $f$. Then $R \cdot f$ is the restriction of $P \circ R^{-1}$. This polynomial is homogeneous of degree $l$, and it is harmonic because $\Delta$ commutes with orthogonal changes of variables. Hence $R \cdot f \in \Harm_l$.

The group $\SO(3)$ acts on $\Harm_l$ via
\[
\begin{split}
\SO(3) \times \Harm_l & \to \Harm_l, \\
( R, f ) & \mapsto R \cdot f.
\end{split}
\]
Then $I \cdot f = f$ and $( R_1 R_2 ) \cdot f = R_1 \cdot ( R_2 \cdot f )$ for all $R_1, R_2 \in \SO(3)$ and $f \in \Harm_l$, and the action is linear in $f$. Hence the map
\[
\rho_l \colon \SO(3) \to \GL ( \Harm_l ),
\]
defined by $\rho_l (R) (f) = R \cdot f$, is a group homomorphism, that is, a representation of $\SO(3)$ on $\Harm_l$. By \cite[Thm.~3.8]{sh-181}, the complexification $\Harm_l \otimes \C$ is equivariantly isomorphic to the space of binary forms of degree $2l$, which is an irreducible $\SL_2 ( \C )$-module. The group $\SO(3)$ acts on it through the covering $\mathrm{SU}(2) \to \SO(3)$, and $\mathrm{SU}(2)$ is Zariski dense in $\SL_2 ( \C )$, so a subspace invariant under $\SO(3)$ is invariant under $\SL_2 ( \C )$. Hence $\Harm_l \otimes \C$ is irreducible, and every $\SO(3)$-equivariant linear map $\Harm_l \to \Harm_l$ is a real multiple of the identity. For $l \neq l'$ the representations on $\Harm_l$ and $\Harm_{l'}$ are not isomorphic, because their dimensions $2l + 1$ and $2l' + 1$ differ.

\begin{rem}[Julia reduction]\label[rem]{rem-julia}
Let $l \geq 1$ and $f \in \Harm_l$. Let $F_f$ be the binary form of degree $2l$ that corresponds to $f$ under the isomorphism above. Assume that $F_f$ has $2l$ distinct roots. Let $J_{F_f}$ be the Julia quadratic of $F_f$, and let $\xi ( J_{F_f} ) \in \mathbb{H}^3$ be its Julia point, as defined in \cite[Sec.~4.1]{MR3579531} and \cite{Beshaj2015}. Then $\xi ( J_{F_f} ) = j$, where $j = (0, 1)$.

Indeed, put $\sigma (z_0, z_1) = ( - \bar{z}_1, \bar{z}_0 )$. The Veronese map $v$ of \cite{sh-181} satisfies $v ( \sigma (z) ) = - \overline{v (z)}$. Since $f$ is real and homogeneous of degree $l$, we have $F_f ( \sigma (z) ) = (-1)^l \, \overline{F_f (z)}$. Hence the set of roots of $F_f$ is invariant under $\alpha \mapsto - 1 / \bar{\alpha}$. The Julia point is $\SL_2 ( \C )$-equivariant by \cite[Thm.~3.4, Lem.~4.4]{MR3579531}. The group $\mathrm{SU}(2)$ fixes $j$, so after a rotation we may assume that no root of $F_f$ is at infinity. By \cite[Thm.~4.1]{MR3579531}, the Julia point is determined by the roots, and conjugating the roots replaces $(z, t)$ by $(\bar{z}, t)$. The map $\alpha \mapsto - 1 / \bar{\alpha}$ extends to the isometry
\[
(z, t) \mapsto \left( \frac{- z}{\abs{z}^2 + t^2} , \frac{t}{\abs{z}^2 + t^2} \right)
\]
of $\mathbb{H}^3$. Its only fixed point is $j$. The stabilizer of $j$ in $\mathrm{PSL}_2 ( \C )$ is $\mathrm{PSU}(2) \cong \SO(3)$. Hence the reduction theory of binary forms in \cite{MR3579531, Beshaj2015} gives no normalization of pose on a single real band.
\end{rem}

Let $\Harm_0$ be the space of constant functions on $\Sph$. Let $d \omega$ be the surface measure on $\Sph$ and $d \mu = d \omega / ( 4 \pi )$ the round probability measure. For real continuous functions $f, g$ on $\Sph$ put
\[
\ip{f}{g} = \int_{\Sph} f g \, d \mu.
\]
The angle brackets thus denote the Euclidean product for vectors and this product for functions; the arguments determine which is meant. The measure $\mu$ is invariant under $\SO(3)$, so $\ip{R \cdot f}{R \cdot g} = \ip{f}{g}$ for every $R \in \SO(3)$. The spaces $\Harm_0, \Harm_1, \Harm_2, \dots$ are pairwise orthogonal for this product; see \cite[Sec.~2.1]{sh-181}. We call $\Harm_l$ the \textbf{band of degree $l$}.

A closed surface $\Sigma \subset \R^3 \setminus \{ 0 \}$ is \textbf{star-shaped} about the origin if every ray $\{ t p : t > 0 \}$ with $p \in \Sph$ meets $\Sigma$ in exactly one point. Let $\Sigma$ be such a surface, and for $p \in \Sph$ let $\rho (p) > 0$ be the unique number with $\rho (p) \, p \in \Sigma$. The function $\rho \colon \Sph \to \R$ is the \textbf{radial function} of $\Sigma$, and $\Sigma = \{ \rho (p) \, p : p \in \Sph \}$. We assume that $\rho$ is continuous. Conversely, every positive continuous function $\rho \colon \Sph \to \R$ is the radial function of the star-shaped surface $\{ \rho (p) \, p : p \in \Sph \}$. This surface is the boundary of the solid $\{ t p : p \in \Sph , \ 0 \leq t \leq \rho (p) \}$, which contains the segment from the origin to each of its points.

Let $L \geq 1$. For $0 \leq l \leq L$ let $f_l$ be the orthogonal projection of $\rho$ onto $\Harm_l$. The truncation of $\rho$ at degree $L$ is
\[
\rho_L = f_0 + f_1 + \dots + f_L .
\]
The band $\Harm_0$ is omitted from the coefficient space below because $\SO(3)$ acts trivially on it. Put
\[
U_L = \bigoplus_{l=1}^{L} \Harm_l , \qquad n = \dim U_L = \sum_{l=1}^{L} ( 2l + 1 ) = L^2 + 2L .
\]
We write $c = ( f_1, \dots, f_L )$ for a point of $U_L$. The point $c$ built from the projections of $\rho$ is the \textbf{coefficient tuple} of $\Sigma$ at degree $L$. The group $\SO(3)$ acts on $U_L$ diagonally by
\[
R \cdot c = ( R \cdot f_1, \dots, R \cdot f_L ).
\]
Rotating the surface by $R \in \SO(3)$ replaces $\rho$ by $R \cdot \rho$, where $( R \cdot \rho ) (p) = \rho ( R^{-1} p )$. This action preserves each $\Harm_l$ and the product $\ip{\cdot}{\cdot}$, so it commutes with the orthogonal projections. Hence the rotation replaces each $f_l$ by $R \cdot f_l$, and the coefficient tuple $c$ by $R \cdot c$. When comparing physical radial solids, the constant $f_0$ must be fixed or retained separately; throughout the volume and clearance statements it is a known constant $\rho_0$.

Let $c \in U_L$. The \textbf{orbit} of $c$ is the set $\SO(3) \cdot c = \{ R \cdot c : R \in \SO(3) \}$. Two points of $U_L$ lie in the same orbit if and only if one is obtained from the other by a rotation. The orbits partition $U_L$, and the set of orbits is denoted $U_L / \SO(3)$.

An inner product on $U_L$ is \textbf{$\SO(3)$-invariant} if $\ip{R \cdot c}{R \cdot c'} = \ip{c}{c'}$ for all $R \in \SO(3)$ and $c, c' \in U_L$. The bands are irreducible, their complexifications are irreducible, and distinct bands are not isomorphic. By Schur's lemma, every $\SO(3)$-invariant inner product on $U_L$ therefore has the form
\[
\ip{c}{c'} = \sum_{l=1}^{L} w_l \ip{f_l}{f_l'},
\]
where $\ip{\cdot}{\cdot}$ on $\Harm_l$ is the product defined above and $w_l > 0$. We fix such an inner product and its norm $\norm{\cdot}$. Let $c, c' \in U_L$. The function $R \mapsto \norm{c - R \cdot c'}$ is continuous on the compact group $\SO(3)$, so it attains its minimum. The \textbf{orbit distance} is
\[
\delta ( c, c' ) = \min_{R \in \SO(3)} \norm{c - R \cdot c'} .
\]
It is a metric on $U_L / \SO(3)$, and $\delta ( c, c' ) = 0$ if and only if $c' \in \SO(3) \cdot c$. It is continuous and semialgebraic on $U_L \times U_L$. These statements are \cite[Thm.~7.3]{sh-181}, whose proof uses only the invariance of the norm and the compactness of $\SO(3)$.

The space $U_L$ is a real vector space of dimension $n$. A map $\varphi \colon U_L \to \R$ is a \textbf{polynomial} if it is a polynomial in the coordinates of $U_L$ with respect to some basis. A change of basis is linear, so this property does not depend on the basis. A map
\[
\Phi = ( \Phi_1, \dots, \Phi_N ) \colon U_L \to \R^N
\]
is a \textbf{polynomial map} if each component $\Phi_j$ is a polynomial. A map $\Phi \colon U_L \to \R^N$ is \textbf{$\SO(3)$-invariant} if
\[
\Phi ( R \cdot c ) = \Phi (c)
\]
for every $R \in \SO(3)$ and every $c \in U_L$. Equivalently, $\Phi$ is constant on every orbit, so it induces a map $U_L / \SO(3) \to \R^N$.

\begin{defn}\label[defn]{def-descriptor}
A \textbf{descriptor} is an $\SO(3)$-invariant polynomial map $\Phi \colon U_L \to \R^N$. Let $c \in U_L$. The \textbf{fiber} of $\Phi$ through $c$ is the set $\Phi^{-1} ( \Phi (c) )$ of all points $x$ of $U_L$ with $\Phi (x) = \Phi (c)$. The descriptor $\Phi$ \textbf{separates orbits} if every fiber of $\Phi$ is a single orbit.
\end{defn}

Every fiber contains the orbit of each of its points, so it is a union of orbits. The descriptor $\Phi$ separates orbits if and only if $\Phi (c) = \Phi (c')$ implies $c' \in \SO(3) \cdot c$. We use the Euclidean norm $\abs{\cdot}$ on $\R^N$.

\begin{defn}\label[defn]{def-radius}
Let $\Phi \colon U_L \to \R^N$ be a descriptor, let $K \subset U_L$ be compact, let $c \in K$, and let $\varepsilon \geq 0$. The \textbf{radius} of $\Phi$ at $c$ with tolerance $\varepsilon$ is
\[
r^{\varepsilon}_{\Phi} ( c ; K ) = \max \{ \delta ( c, c' ) : c' \in K , \ \abs{\Phi (c') - \Phi (c)} \leq \varepsilon \} .
\]
We write $r_{\Phi} ( c ; K ) = r^{0}_{\Phi} ( c ; K )$.
\end{defn}

The maximum exists. The set over which it is taken is closed in $K$, hence compact, and it contains $c$. The function $\delta$ is continuous. The radius is nondecreasing in $\varepsilon$. By the properties of $\delta$, $r_{\Phi} ( c ; K ) = 0$ if and only if the fiber of $\Phi$ through $c$ meets $K$ only in $\SO(3) \cdot c$.

The radius has two interpretations. In the first, an adversary builds decoys against a recognizer. The recognizer holds $\Phi (c)$ for the coefficient tuple $c$ of an enrolled surface and accepts $c'$ when
\[
\abs{\Phi (c') - \Phi (c)} \leq \varepsilon.
\]
A \textbf{decoy} for $c$ is an accepted point $c' \notin \SO(3) \cdot c$. Then $r^{\varepsilon}_{\Phi} ( c ; K )$ is the largest orbit distance from $c$ of a decoy in $K$, and it is zero if there is none. We call it the \textbf{decoy radius}. In the second interpretation, an adversary has obtained $\Phi (c)$ and knows that $c \in K$. Every point of the fiber of $\Phi$ through $c$ that lies in $K$ is consistent with this information, and $r_{\Phi} ( c ; K )$ is the largest orbit distance from $c$ among them. A large radius serves the first adversary. A small radius serves the second. \Cref{fig-power-decoy} shows two surfaces whose band powers coincide. Each is a decoy for the other under the descriptor $P = ( P_1, P_2 )$ of \cref{sec-powers}.

For comparison across components of different degrees, fix a nonsingular matrix $W \in \R^{N \times N}$ and use the discrepancy $\abs{W ( \Phi(c') - \Phi(c) )}$. This is equivalent to replacing $\Phi$ by the descriptor $W\Phi$. It leaves exact fibers and Jacobian ranks unchanged, but in general it changes positive-tolerance radii. The matrix must be fixed when the tolerance is specified. In particular, a componentwise relative error corresponds to
\[
W = \diag ( 1 / \abs{\Phi_1 (c)}, \dots, 1 / \abs{\Phi_N (c)} ),
\]
which depends on $c$ and is undefined when a component vanishes; it requires a positive absolute floor $\eta_j$, that is, the entries $1 / \max \{ \abs{\Phi_j (c)}, \eta_j \}$.

\begin{defn}\label[defn]{def-observation-fiber}
Let $\Phi \colon U_L \to \R^N$ be a descriptor, let $W \in \R^{N \times N}$ be nonsingular, let $K \subset U_L$ be compact, let $y \in \R^N$, and let $\varepsilon \geq 0$. The \textbf{observation-consistent set} is
\[
\mathcal F^{\varepsilon}_{\Phi,W}(y;K)
= \{ x \in K : \abs{W(\Phi(x)-y)} \leq \varepsilon \}.
\]
When this set is nonempty, its \textbf{orbit diameter} is
\[
D^{\varepsilon}_{\Phi,W}(y;K)
= \max_{x,x'\in\mathcal F^{\varepsilon}_{\Phi,W}(y;K)}\delta(x,x').
\]
\end{defn}

Let $c \in K$. If $y=\Phi(c)$, the largest distance from $c$ in this set is $r^{\varepsilon}_{W\Phi}(c;K)$. If both
\[
\abs{W(y-\Phi(c))}\leq\varepsilon \quad \text{and} \quad x\in\mathcal F^{\varepsilon}_{\Phi,W}(y;K),
\]
then $\abs{W(\Phi(x)-\Phi(c))}\leq 2\varepsilon$ by the triangle inequality. An empty consistent set signals incompatibility among the observation, the error budget, and $K$; it is not a certificate about an enrolled shape. All information statements below are conditional on the specified $K$. A catalog or a geometric prior can make the part of a fiber that lies in $K$ much smaller than the whole fiber.

\section{Separating descriptors disclose}\label{sec-disclosure}

Separating descriptors exist on every $U_L$. Real forms of the generators of the invariant ring separate the $\SO(3)$-orbits by \cite[Thm.~4.1]{sh-181}. Recovery of an orbit from the values of invariants, estimated from noisy samples, is the subject of \cite{BandeiraEtAl23}. Let $I$ be a separating descriptor and let $K \subset U_L$ be compact and semialgebraic. By \cite[Thm.~7.3(iii)]{sh-181} there are constants $C_K > 0$ and $\alpha_K \in (0, 1]$ such that
\[
\delta ( c, c' ) \leq C_K \, \abs{I (c) - I (c')}^{\alpha_K} , \qquad c, c' \in K .
\]
The proof is the \L{}ojasiewicz inequality for the continuous semialgebraic functions $\delta ( c, c' )$ and $\abs{I (c) - I (c')}$ on $K \times K$, which have the same zero set. It applies verbatim on $U_L$ with any invariant inner product. In \cite{sh-181} the inequality certifies the descriptor. The same argument bounds the positive-tolerance radius of every descriptor by its exact radius plus a H\"older error term.

\begin{prop}\label[prop]{prop-tolerance}
Let $\Phi \colon U_L \to \R^N$ be a descriptor, let $K \subset U_L$ be compact and semialgebraic, and let $c \in K$. Then there are $C > 0$ and $\alpha \in (0, 1]$ such that
\[
r^{\varepsilon}_{\Phi} ( c ; K ) \leq r_{\Phi} ( c ; K ) + C \varepsilon^{\alpha}
\]
for every $\varepsilon \geq 0$.
\end{prop}

\begin{proof}
Let $F$ be the set of $x \in K$ with $\Phi (x) = \Phi (c)$. It is compact and semialgebraic, and it contains $c$. The functions $d (x) = \min_{y \in F} \norm{x - y}$ and $h (x) = \abs{\Phi (x) - \Phi (c)}$ are continuous and semialgebraic on $K$, with the same zero set $F$. The \L{}ojasiewicz inequality in \cite{BochnakCosteRoy98} gives $C > 0$ and $\alpha \in (0, 1]$ with $d (x) \leq C h (x)^{\alpha}$ on $K$. Let $c' \in K$ with $h (c') \leq \varepsilon$, and choose $x \in F$ with $\norm{c' - x} = d (c')$. Then
\[
\delta ( c, c' ) \leq \delta ( c, x ) + \delta ( x, c' ) \leq r_{\Phi} ( c ; K ) + \norm{x - c'} \leq r_{\Phi} ( c ; K ) + C \varepsilon^{\alpha} .
\]
Taking the maximum over $c'$ proves the claim.
\end{proof}

The constants depend on $c$. For a separating descriptor $r_{\Phi} ( c ; K ) = 0$, and the inequality of \cite{sh-181} makes the constants uniform on $K$. The next proposition gives the corresponding bounds under both interpretations.

\begin{prop}\label[prop]{prop-disclosure}
Let $I \colon U_L \to \R^N$ be a separating descriptor. Let $K \subset U_L$ be compact and semialgebraic, and let $C_K$ and $\alpha_K$ be as above. Then the following hold.
\begin{enumerate}
\item For every $c \in K$ and every $\varepsilon \geq 0$, $r^{\varepsilon}_{I} ( c ; K ) \leq C_K \, \varepsilon^{\alpha_K}$. In particular $r_{I} ( c ; K ) = 0$.
\item Let $c \in K$, let $\varepsilon \geq 0$, and let $y \in \R^N$ satisfy $\abs{y - I (c)} \leq \varepsilon$. Then every $c' \in K$ with $\abs{I (c') - y} \leq \varepsilon$ satisfies $\delta ( c, c' ) \leq C_K \, ( 2 \varepsilon )^{\alpha_K}$.
\end{enumerate}
\end{prop}

\begin{proof}
Let $c' \in K$ with $\abs{I (c') - I (c)} \leq \varepsilon$. The displayed inequality gives $\delta ( c, c' ) \leq C_K \, \varepsilon^{\alpha_K}$. Taking the maximum over $c'$ proves (1). For $\varepsilon = 0$ the bound is zero. In (2) the triangle inequality gives $\abs{I (c') - I (c)} \leq 2 \varepsilon$, and (1) applies with $2 \varepsilon$.
\end{proof}

Part (1) says that a separating descriptor admits no decoys at tolerance zero, and only decoys within $C_K \varepsilon^{\alpha_K}$ at tolerance $\varepsilon$. Part (2) gives the disclosure interpretation: if a released value $y$ is known with error $\varepsilon$, then the enrolled tuple lies within orbit distance $C_K ( 2 \varepsilon )^{\alpha_K}$ of every point of $K$ whose descriptor lies within $\varepsilon$ of $y$. That set of points contains $c$, and membership in it is decided by evaluating $I$. The differential of $I$ annihilates the orbit directions, so its rank never exceeds $n - 3$ on the free locus. Where the Jacobian of $I$ restricted to a slice transverse to the orbit has full rank $n - 3$ on the principal stratum, the exponent can be taken equal to one locally \cite[Thm.~7.3(iv)]{sh-181}.

On $U_2$ the disclosure is explicit. Identify $\Harm_1$ with $\R^3$ and $\Harm_2$ with the traceless symmetric $3 \times 3$ matrices by $f_1 (x) = \ip{a}{x}$ and $f_2 (x) = x^\top Q x$. Then $R \cdot ( a, Q ) = ( R a, R Q R^\top )$. By \cite[Thm.~4.9]{sh-181} the invariants
\[
\abs{a}^2 , \quad \tr Q^2 , \quad \det Q , \quad a^\top Q a , \quad a^\top Q^2 a , \quad \det ( a, Q a, Q^2 a )
\]
separate the $\SO(3)$-orbits in $U_2$. The proof in \cite{sh-181} is constructive and yields an explicit reconstruction. The eigenvalues of $Q$ are the roots of $t^3 - \frac{1}{2} ( \tr Q^2 ) \, t - \det Q$. Let $a_1, a_2, a_3$ be the coordinates of $a$ in a positively oriented orthonormal eigenbasis of $Q$, with eigenvalues $\lambda_1, \lambda_2, \lambda_3$. Then
\[
\det ( a, Q a, Q^2 a ) = a_1 a_2 a_3 \prod_{i < j} ( \lambda_j - \lambda_i ) .
\]
When the eigenvalues are distinct, the squares $a_i^2$ solve the Vandermonde system $\sum_{i} \lambda_i^k a_i^2 = a^\top Q^k a$ for $k = 0, 1, 2$, and the displayed identity fixes the sign of $a_1 a_2 a_3$. Thus, on the locus where $Q$ has distinct eigenvalues, the six numbers determine $( a, Q )$ up to rotation by one cubic equation and one linear system. The cases of a repeated eigenvalue are treated separately in the proof of \cite[Thm.~4.9]{sh-181}, and separation holds on all of $U_2$.

\begin{prop}[Reconstruction from a released value]\label[prop]{prop-minimax-shape}
Let $\mathcal F=\mathcal F^{\varepsilon}_{\Phi,W}(y;K)$ be nonempty and put $D=D^{\varepsilon}_{\Phi,W}(y;K)$. Define the deterministic worst-case reconstruction error by
\[
\mathcal R^*(y;K)
=\min_{\widehat c\in K}\ \max_{x\in\mathcal F}\delta(\widehat c,x).
\]
Then $D/2\leq\mathcal R^*(y;K)\leq D$. For every $c\in\mathcal F$, if $r_c=\max_{x\in\mathcal F}\delta(c,x)$, then $r_c\leq D\leq 2r_c$.
\end{prop}

\begin{proof}
Compactness and continuity give all maxima. The function of $\widehat c$ being minimized is continuous, since its change is at most $\delta(\widehat c,\widehat c')\leq\norm{\widehat c-\widehat c'}$, so the minimum also exists. For any $x,x'\in\mathcal F$ and $\widehat c\in K$, the triangle inequality gives $\delta(x,x')\leq\delta(x,\widehat c)+\delta(\widehat c,x')$. Taking maxima proves the lower bound. Taking $\widehat c$ to be any point of $\mathcal F$ proves the upper bound and $r_c\leq D$. The same triangle inequality through $c$ gives $D\leq 2r_c$.
\end{proof}

This is an information bound. It provides neither a reconstruction algorithm nor a probabilistic privacy guarantee: it assumes no distribution on $K$ and gives no bound on the computational effort needed to find a consistent point. A large shape diameter need not hide a particular attribute; \cref{sec-tasks,sec-volume} make that distinction explicit.

\section{Fibers of descriptors}\label{sec-fibers}

Let $L \geq 2$. For $c \in U_L$ let $( a, Q )$ correspond to $( f_1, f_2 )$ as in \cref{sec-disclosure}, and put
\[
\Omega_L = \{ c \in U_L : \det ( a, Q a, Q^2 a ) \neq 0 \} .
\]

\begin{lem}\label[lem]{lem-free}
Let $L \geq 2$. The set $\Omega_L$ is open, dense, and $\SO(3)$-invariant, and $\SO(3)$ acts freely on it.
\end{lem}

\begin{proof}
For $R \in \SO(3)$,
\[
\det ( R a, R Q R^\top R a, R Q^2 R^\top R a ) = \det R \cdot \det ( a, Q a, Q^2 a ) = \det ( a, Q a, Q^2 a ) .
\]
So $\Omega_L$ is invariant. The polynomial $\det ( a, Q a, Q^2 a )$ is not zero: for $a = (1, 1, 1)$ and $Q = \diag (1, 0, -1)$ it equals $-2$. So $\Omega_L$ is the complement of the zero set of a nonzero polynomial, and it is open and dense. Let $c \in \Omega_L$, and let $\lambda_i$ and $a_i$ be as in \cref{sec-disclosure}. The identity $\det ( a, Q a, Q^2 a ) = a_1 a_2 a_3 \prod_{i<j} ( \lambda_j - \lambda_i )$ shows that the $\lambda_i$ are distinct and that every $a_i$ is nonzero. Let $R \cdot c = c$. Then $R Q R^\top = Q$, so $R$ preserves each eigenline of $Q$. In the eigenbasis $R$ is diagonal with entries $\pm 1$. Since $R a = a$ and every $a_i$ is nonzero, all entries equal $1$. Hence $R = I$.
\end{proof}

\begin{thm}\label[thm]{thm-fiber}
Let $L \geq 2$ and let $\Phi \colon U_L \to \R^N$ be a descriptor. Let $r$ be the maximum over $U_L$ of the rank of $d \Phi$, and let $U' = \{ c \in \Omega_L : \rank d \Phi_c = r \}$. Then the following hold.
\begin{enumerate}
\item The set $U'$ is open, dense, and $\SO(3)$-invariant, and $r \leq \min \{ N, n - 3 \}$.
\item For every $c \in U'$ the set $M_c = \Phi^{-1} ( \Phi (c) ) \cap U'$ is an $\SO(3)$-invariant embedded submanifold of $U_L$ of dimension $n - r$, and $M_c / \SO(3)$ is a smooth manifold of dimension $n - 3 - r$.
\item If $r < n - 3$, then $r_{\Phi} ( c ; K ) > 0$ for every $c \in U'$ and every compact neighborhood $K$ of $c$.
\item If $r = n - 3$, then $M_c$ is a finite union of $\SO(3)$-orbits for every $c \in U'$.
\end{enumerate}
\end{thm}

\begin{proof}
(1) If $r = 0$, then $\Phi$ is constant, the set where $d \Phi$ has rank $r$ is $U_L$, and $U' = \Omega_L$. Assume $r > 0$. The Jacobian of $\Phi$ has rank at most $r$ everywhere. So the set where it has rank $r$ is the complement of the common zero set of its $r \times r$ minors. These minors are polynomials, and one of them is not identically zero. Hence this set is a nonempty Zariski open subset of $U_L$. By \cref{lem-free} so is $\Omega_L$. Their intersection $U'$ is nonempty because $U_L$ is irreducible, so it is open and dense. Invariance of $\Phi$ gives $d \Phi_{R \cdot c} \circ R = d \Phi_c$. So the rank of $d \Phi$ is constant on orbits, and $U'$ is invariant. Clearly $r \leq N$. Let $c \in U'$. The action is free at $c$, so the orbit $\SO(3) \cdot c$ is an embedded submanifold of dimension $3$. The map $\Phi$ is constant on it, so its tangent space at $c$ lies in the kernel of $d \Phi_c$. Hence $r \leq n - 3$.

(2) The map $\Phi$ has constant rank $r$ on the open set $U'$. By the constant-rank level set theorem in \cite{Lee13}, $M_c$ is a properly embedded submanifold of $U'$ of dimension $n - r$. It is invariant because $\Phi$ and $U'$ are. The compact group $\SO(3)$ acts smoothly and freely on $M_c$. By the quotient manifold theorem in \cite{Lee13}, $M_c / \SO(3)$ is a smooth manifold of dimension $n - r - 3$.

(3) Let $r < n - 3$, and let $K$ be a compact neighborhood of $c$. Choose $\eta > 0$ such that the closed ball of radius $\eta$ about $c$ lies in $K$. The intersection $\mathcal O$ of $M_c$ with the open ball of radius $\eta$ about $c$ is a nonempty open subset of $M_c$. It has dimension $n - r > 3$. The orbit $\SO(3) \cdot c$ is a submanifold of dimension $3$, so it has measure zero in $M_c$ and does not contain $\mathcal O$. Choose $c' \in \mathcal O$ outside $\SO(3) \cdot c$. Then $c' \in K$, $\Phi (c') = \Phi (c)$, and $\delta ( c, c' ) > 0$. Hence $r_{\Phi} ( c ; K ) \geq \delta ( c, c' ) > 0$.

(4) Let $r = n - 3$. Then $M_c$ is a manifold of dimension $3$. Every orbit in $M_c$ is a compact embedded submanifold of dimension $3$, so it is open and closed in $M_c$. Since $\SO(3)$ is connected, the orbits in $M_c$ are its connected components. The set $U'$ is semialgebraic: it is the intersection of $\Omega_L$ with the set where the sum of the squares of the $r \times r$ minors is nonzero. So $M_c$ is semialgebraic, and it has finitely many connected components; see \cite{BochnakCosteRoy98}. Hence $M_c$ is a finite union of orbits.
\end{proof}

We call $U'$ the \textbf{free maximal-rank locus} of $\Phi$. For $c \in U'$ we call $M_c$ the \textbf{regular fiber} of $\Phi$ through $c$.

\begin{cor}\label[cor]{cor-count}
Let $L \geq 2$ and let $\Phi \colon U_L \to \R^N$ be a descriptor with $N < n - 3$. Then $\Phi$ does not separate orbits. Moreover, there is an open dense set $U' \subset U_L$ such that $r_{\Phi} ( c ; K ) > 0$ for every $c \in U'$ and every compact neighborhood $K$ of $c$.
\end{cor}

\begin{proof}
By \cref{thm-fiber}, $r \leq N < n - 3$. Part (3) of \cref{thm-fiber} gives the second claim. It also gives a point of a fiber outside the orbit of $c$, which proves the first claim.
\end{proof}

Only the number of components of $\Phi$ enters the corollary; the components themselves play no role. A recognizer built on fewer than $n - 3$ invariants admits decoys near every generic enrolled surface, at every tolerance.

\begin{cor}[Combined releases]\label[cor]{cor-composition}
Let $L \geq 2$, let $\Phi_1,\dots,\Phi_m$ be descriptors on $U_L$, and let $\Psi=(\Phi_1,\dots,\Phi_m)$. Then
\[
\Psi^{-1}(\Psi(c))=\bigcap_{j=1}^m\Phi_j^{-1}(\Phi_j(c)).
\]
On the open dense free maximal-rank locus for $\Psi$, the regular fiber modulo rotations has dimension $n-3-r_{\Psi}$, where
\[
r_{\Psi}=\max_{c\in U_L}\rank
\begin{pmatrix}
d(\Phi_1)_c\\[-1mm] \vdots\\[-1mm]d(\Phi_m)_c
\end{pmatrix}.
\]
For separately specified observation error budgets, the jointly consistent set is the intersection of the individual consistent sets, and when nonempty, its diameter cannot exceed any individual diameter.
\end{cor}

\begin{proof}
Equality of the concatenated values is equivalent to equality of each component value, and the displayed matrix is $d\Psi_c$. Apply \cref{thm-fiber}. The last assertion follows from set inclusion and from the monotonicity of the maximum that defines the diameter.
\end{proof}

Thus separate releases about the same tuple must be audited jointly. Their ranks need not add, because their differentials can be dependent. The conclusion applies to releases about one underlying tuple. It does not apply to measurements of unrelated objects.

\section{Band powers}\label{sec-powers}

The band powers $P_l (c) = \norm{f_l}^2$ form the descriptor $P = ( P_1, \dots, P_L )$ of \cite{KazhdanFunkhouserRusinkiewicz03}. The fiber of $P$ through $c$ is a product of spheres. It contains $( R_1 \cdot f_1, \dots, R_L \cdot f_L )$ for all $R_1, \dots, R_L \in \SO(3)$. The gradient of $P_l$ is a nonzero multiple of $f_l$, so $P$ has rank $L$ wherever every $f_l$ is nonzero. By \cref{thm-fiber}, for $L \geq 2$ the fiber of $P$ through a generic point is, modulo rotations, a manifold of dimension $L^2 + L - 3$.

On $\Harm_2$ we use the matrix $Q$ of \cref{sec-disclosure} and the Frobenius norm $\norm{Q}$. The invariant inner product on $\Harm_2$ is unique up to a positive factor. Any other choice multiplies all radii on $\Harm_2$ by one constant. In the next two statements the definitions of \cref{sec-radii} are applied with $\Harm_2$ in place of $U_L$; the orbit distance on $\Harm_2$ is the restriction of the orbit distance on $U_2$ to the pairs with $a = a' = 0$.

\begin{prop}\label[prop]{prop-quadrupole}
Let $r > 0$ and let $Q \in \Harm_2$ with $\norm{Q} = r$. Let $\theta \in [0, \pi/3]$ satisfy $\cos 3 \theta = 3 \sqrt{6} \, \det Q / r^3$. Let $K \subset \Harm_2$ be compact and contain the sphere $\{ Q' : \norm{Q'} = r \}$. Then
\[
r_{P_2} ( Q ; K ) = 2 r \sin \frac{ \max \{ \theta, \pi/3 - \theta \} }{2} .
\]
In particular $\frac{\sqrt{6} - \sqrt{2}}{2} \, r \leq r_{P_2} ( Q ; K ) \leq r$. The lower bound is attained if and only if $\det Q = 0$. The upper bound is attained if and only if $Q$ has a repeated eigenvalue.
\end{prop}

\begin{proof}
For $\varphi \in \R$ put
\[
\mu ( \varphi ) = r \sqrt{2/3} \, \bigl( \cos \varphi , \ \cos ( 2\pi/3 - \varphi ) , \ \cos ( 2\pi/3 + \varphi ) \bigr) .
\]
The identity $\sum_{k=0}^{2} \cos ( x + 2 \pi k / 3 ) \cos ( y + 2 \pi k / 3 ) = \frac{3}{2} \cos ( x - y )$ gives $\ip{\mu ( \varphi )}{\mu ( \varphi' )} = r^2 \cos ( \varphi - \varphi' )$. Hence
\[
\norm{\mu ( \varphi ) - \mu ( \varphi' )} = 2 r \sin \frac{ \abs{\varphi - \varphi'} }{2} , \qquad \abs{\varphi - \varphi'} \leq \pi .
\]
So $\mu$ traces the circle of radius $r$ in the plane of vectors with coordinate sum zero, at unit angular speed. The coordinates of $\mu ( \varphi )$ are in decreasing order exactly on the arc $\varphi \in [0, \pi/3]$. So the spectra, in decreasing order, of the matrices $Q' \in \Harm_2$ with $\norm{Q'} = r$ are exactly the vectors $\mu ( \varphi )$ with $\varphi \in [0, \pi/3]$.

Let $\lambda (A)$ denote the spectrum of a symmetric matrix $A$ in decreasing order. By the Hoffman--Wielandt inequality in \cite{HoffmanWielandt53} and the rearrangement inequality, $\norm{Q - U Q' U^\top} \geq \norm{\lambda (Q) - \lambda (Q')}$ for every $U \in \Or(3)$. Equality holds when $U$ carries an orthonormal eigenbasis of $Q'$ to one of $Q$ with the eigenvalues in the same order. Replacing $U$ by $-U$ does not change $U Q' U^\top$, so the minimum over $\SO(3)$ equals the minimum over $\Or(3)$. Hence $\delta ( Q, Q' ) = \norm{\lambda (Q) - \lambda (Q')}$.

Write $\lambda (Q) = \mu ( \theta' )$ with $\theta' \in [0, \pi/3]$. The identity $\prod_{k=0}^{2} \cos ( \theta' + 2 \pi k / 3 ) = \frac{1}{4} \cos 3 \theta'$ gives
\[
\det Q = r^3 \Bigl( \frac{2}{3} \Bigr)^{3/2} \frac{ \cos 3 \theta' }{4} = \frac{ r^3 \cos 3 \theta' }{ 3 \sqrt{6} } .
\]
Since $3 \theta' \in [0, \pi]$, this determines $\theta'$, and $\theta' = \theta$. The fiber of $P_2$ through $Q$ is the sphere $\{ \norm{Q'} = r \}$, which lies in $K$. Therefore
\[
r_{P_2} ( Q ; K ) = \max_{ \varphi \in [0, \pi/3] } 2 r \sin \frac{ \abs{\theta - \varphi} }{2} = 2 r \sin \frac{ \max \{ \theta, \pi/3 - \theta \} }{2} ,
\]
because $\sin$ is increasing on $[0, \pi/6]$. The quantity $\max \{ \theta, \pi/3 - \theta \}$ lies in $[\pi/6, \pi/3]$. It equals $\pi/6$ exactly when $\theta = \pi/6$, that is, when $\det Q = 0$. It equals $\pi/3$ exactly when $\theta \in \{ 0, \pi/3 \}$, that is, when two eigenvalues of $Q$ coincide. The values $2 r \sin ( \pi/12 ) = \frac{\sqrt{6} - \sqrt{2}}{2} \, r$ and $2 r \sin ( \pi/6 ) = r$ complete the proof.
\end{proof}

Adjoining $\det Q$ to $P_2$ gives a separating descriptor on $\Harm_2$. The numbers $\tr Q^2$ and $\det Q$ determine the spectrum of $Q$, and symmetric matrices with the same spectrum are conjugate by a rotation. So the radius of $( P_2, \det Q )$ vanishes at every point. The invariant cubic forms on $\Harm_2$ are the multiples of $\det Q$. So $\det Q$ is, up to a nonzero factor, the bispectrum of the quadrupole with itself in the sense of \cite{Kakarala12}.

At positive tolerance the radius of the separating pair is again explicit. Knowledge of the power is side information, recorded by the set $K$.

\begin{cor}\label[cor]{cor-quadrupole-tolerance}
Let $r > 0$, let $S_r = \{ Q' \in \Harm_2 : \norm{Q'} = r \}$, and let $I = ( P_2, \det )$. Let $Q \in S_r$, let $\theta$ be as in \cref{prop-quadrupole}, let $\varepsilon \geq 0$, and put $\eta = 3 \sqrt{6} \, \varepsilon / r^3$ and
\[
\begin{split}
\varphi_- & = \tfrac{1}{3} \arccos \bigl( \min \{ 1, \cos 3 \theta + \eta \} \bigr) , \\
\varphi_+ & = \tfrac{1}{3} \arccos \bigl( \max \{ -1, \cos 3 \theta - \eta \} \bigr) .
\end{split}
\]
Then
\[
r^{\varepsilon}_{I} ( Q ; S_r ) = 2 r \sin \frac{ \max \{ \theta - \varphi_- , \ \varphi_+ - \theta \} }{2} .
\]
As $\varepsilon \to 0$, $r^{\varepsilon}_{I} ( Q ; S_r ) = \frac{ r \eta }{ 3 \sin 3 \theta } + O ( \eta^2 )$ if $0 < \theta < \pi/3$, and $r^{\varepsilon}_{I} ( Q ; S_r ) = \frac{ r \sqrt{2 \eta} }{3} + O ( \eta^{3/2} )$ if $\theta \in \{ 0, \pi/3 \}$.
\end{cor}

\begin{proof}
For $Q' \in S_r$ write $\lambda (Q') = \mu ( \varphi )$ with $\varphi \in [0, \pi/3]$, as in the proof of \cref{prop-quadrupole}. The power is constant on $S_r$, and $\det Q' = r^3 \cos 3 \varphi / ( 3 \sqrt{6} )$. So $Q'$ is accepted if and only if $\abs{ \cos 3 \varphi - \cos 3 \theta } \leq \eta$. Since $\cos 3 \varphi$ decreases on $[0, \pi/3]$, this holds if and only if $\varphi \in [ \varphi_-, \varphi_+ ]$. By the proof of \cref{prop-quadrupole}, $\delta ( Q, Q' ) = 2 r \sin ( \abs{ \theta - \varphi } / 2 )$. Its maximum over $[ \varphi_-, \varphi_+ ]$ is the stated value. For $0 < \theta < \pi/3$ the derivative of $\arccos$ at $\cos 3 \theta$ is $-1 / \sin 3 \theta$, which gives $\varphi_{\pm} = \theta \pm \eta / ( 3 \sin 3 \theta ) + O ( \eta^2 )$. For $\theta = 0$ we have $\varphi_- = 0$ and $\varphi_+ = \frac{1}{3} \arccos ( 1 - \eta ) = \frac{1}{3} \sqrt{2 \eta} + O ( \eta^{3/2} )$. The case $\theta = \pi/3$ is symmetric.
\end{proof}

The maximum over the fixed-power sphere is also explicit:
\begin{equation}\label{eq-quadrupole-uniform}
\max_{Q\in S_r}r^{\varepsilon}_{I}(Q;S_r)
=2r\sin\left(\frac16\arccos\bigl(1-\min\{\eta,2\}\bigr)\right).
\end{equation}
To see this, for $0\leq a\leq b\leq\pi$ the identity $\cos a-\cos b=2\sin((a+b)/2)\sin((b-a)/2)$ gives $\cos a-\cos b\geq1-\cos(b-a)$. Apply this with $a,b$ the ordered angles $3\theta,3\varphi$; equality is attained when one angle is an endpoint. Thus the uniform exponent $1/2$ is attained and cannot be increased. At a fixed spectrum with $0 < \theta < \pi/3$ the radius is linear to first order. These statements assume exact knowledge of the power and do not describe joint noise in both invariants.

\begin{prop}\label[prop]{prop-dipole-quadrupole}
Let $L = 2$ and use on $U_2$ the invariant norm $\norm{( a, Q )}^2 = \abs{a}^2 + \norm{Q}^2$. Let $r_1, r_2 > 0$, let $F = \{ ( a, Q ) : \abs{a} = r_1 , \ \norm{Q} = r_2 \}$ be the corresponding fiber of $P = ( P_1, P_2 )$, and put
\[
c_{\pm} = \Bigl( r_1 e_1 , \ \pm \frac{r_2}{\sqrt{6}} \diag ( 2, -1, -1 ) \Bigr) .
\]
Then $c_{\pm} \in F$ and
\[
\delta ( c_+ , c_- )^2 =
\begin{cases}
2 r_1^2 + r_2^2 - \dfrac{ r_1^4 }{ 3 r_2^2 } , & r_1^2 \leq 3 r_2^2 , \\[2mm]
4 r_2^2 , & r_1^2 \geq 3 r_2^2 .
\end{cases}
\]
Consequently $r_P ( c ; K ) \geq \frac{1}{2} \, \delta ( c_+ , c_- )$ for every $c \in F$ and every compact $K \supset F$.
\end{prop}

\begin{proof}
Put $s = r_2 / \sqrt{6}$ and $E = 3 e_1 e_1^\top - I = \diag ( 2, -1, -1 )$, so $c_{\pm} = ( r_1 e_1 , \pm s E )$. Both points lie in $F$ because $\norm{E}^2 = 6$. Let $R \in \SO(3)$, put $u = R e_1$, and put $t = \ip{u}{e_1}$. Then
\[
\begin{split}
\norm{r_1 e_1 - R ( r_1 e_1 )}^2 & = 2 r_1^2 ( 1 - t ) , \\
s E - R ( - s E ) R^\top & = s \bigl( 3 e_1 e_1^\top + 3 u u^\top - 2 I \bigr) .
\end{split}
\]
The symmetric matrix $e_1 e_1^\top + u u^\top$ has eigenvalues $1 + t$, $1 - t$, and $0$. So the second matrix has eigenvalues $s ( 1 + 3 t )$, $s ( 1 - 3 t )$, and $- 2 s$, and its squared norm is $s^2 ( 6 + 18 t^2 ) = r_2^2 ( 1 + 3 t^2 )$. Hence
\[
\norm{c_+ - R \cdot c_-}^2 = g (t) , \qquad g (t) = 2 r_1^2 ( 1 - t ) + r_2^2 ( 1 + 3 t^2 ) .
\]
Every $t \in [-1, 1]$ occurs. The function $g$ is convex, with minimum at $t_0 = r_1^2 / ( 3 r_2^2 ) > 0$. If $t_0 \leq 1$, then $\delta ( c_+ , c_- )^2 = g ( t_0 ) = 2 r_1^2 + r_2^2 - r_1^4 / ( 3 r_2^2 )$. If $t_0 \geq 1$, then $g$ is decreasing on $[-1, 1]$, and $\delta ( c_+ , c_- )^2 = g (1) = 4 r_2^2$. For $c \in F$ the triangle inequality gives $\max \{ \delta ( c, c_+ ), \delta ( c, c_- ) \} \geq \frac{1}{2} \, \delta ( c_+ , c_- )$, and both $c_{\pm}$ lie in $F \subset K$.
\end{proof}

For the norm $w_1 \abs{a}^2 + w_2 \norm{Q}^2$, replace $r_1$ and $r_2$ by $\sqrt{w_1} \, r_1$ and $\sqrt{w_2} \, r_2$. The map $( a, Q ) \mapsto ( \sqrt{w_1} \, a, \sqrt{w_2} \, Q )$ is an equivariant isometry onto $U_2$ with the norm of \cref{prop-dipole-quadrupole}, and it maps fibers of $P$ to fibers of $P$.

\begin{figure}[ht]
\centering
\includegraphics[width=0.8\textwidth]{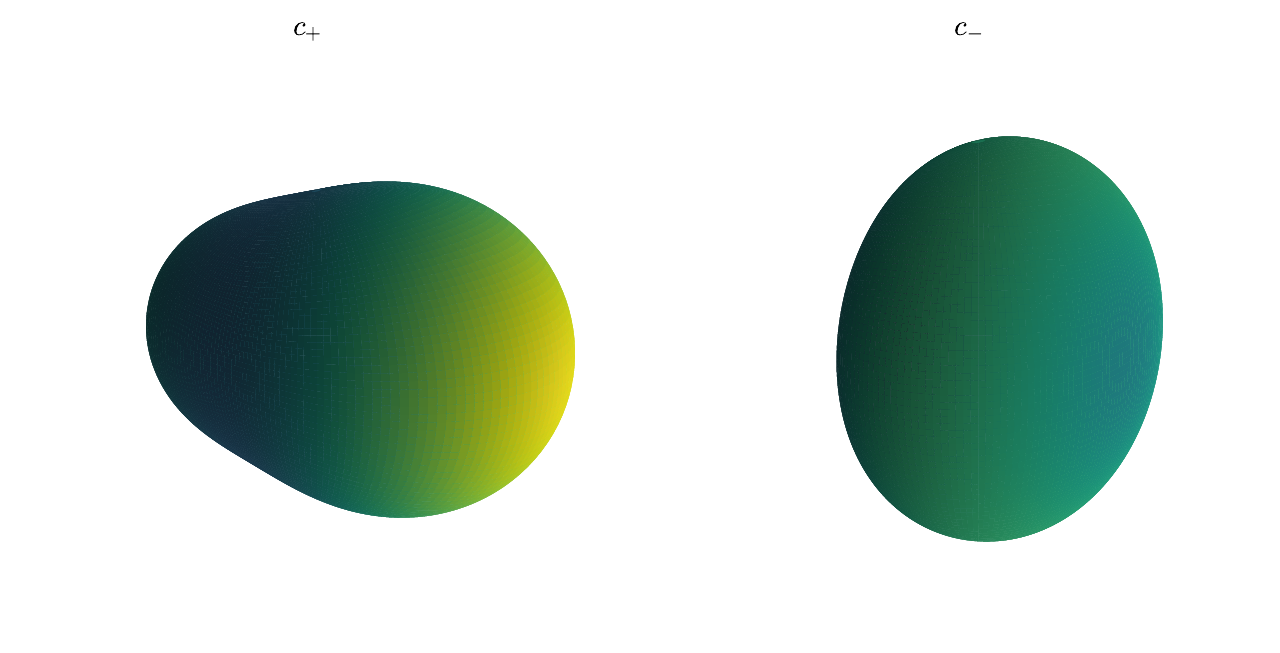}
\caption{The surfaces $c_+$ (left) and $c_-$ (right) of \cref{prop-dipole-quadrupole}, drawn as radial surfaces $\rho_0 + f_1 + f_2$ with $\rho_0 = 1$, $r_1 = 0.25$, and $r_2 = 0.35$, where $f_1 (p) = \ip{a}{p}$ and $f_2 (p) = p^\top Q p$. Their band powers coincide. Their orbit distance is $\delta ( c_+, c_- ) \approx 0.487$.}
\label{fig-power-decoy}
\end{figure}

\section{The spectral pool}\label{sec-pool}

For binary forms $F$ and $G$ of degrees $m$ and $m'$ in $( z_0, z_1 )$, the transvectant of order $k$ is
\[
( F, G )_k = \frac{ (m - k)! \, (m' - k)! }{ m! \, m'! } \sum_{i=0}^{k} (-1)^i \binom{k}{i} \frac{ \partial^k F }{ \partial z_0^{k-i} \, \partial z_1^{i} } \, \frac{ \partial^k G }{ \partial z_0^{i} \, \partial z_1^{k-i} } ,
\]
as in \cite[Eq.~(18)]{sh-181}. By \cite[Thm.~3.8]{sh-181}, restriction to the null cone gives an $\SL_2 ( \C )$-equivariant linear isomorphism $\Psi_l$ from $\Harm_l \otimes \C$ onto the binary forms of degree $2l$. Write $F_l = \Psi_l ( f_l )$ and
\[
\begin{split}
\cA (L) & = \{ ( l_1, l_2, l_3 ) \in \Z^3 : 1 \leq l_1 \leq l_2 \leq l_3 \leq L , \ l_3 \leq l_1 + l_2 , \ l_1 + l_2 + l_3 \ \text{even} \} , \\
\cA^- (L) & = \{ ( l_1, l_2, l_3 ) \in \Z^3 : 1 \leq l_1 < l_2 < l_3 \leq L , \ l_3 \leq l_1 + l_2 , \ l_1 + l_2 + l_3 \ \text{odd} \} .
\end{split}
\]
For $L \geq 2$ the pool $\Phi_L$ has the following components.
\begin{enumerate}
\item The powers $P_l = ( F_l, F_l )_{2l}$ for $1 \leq l \leq L$.
\item The bispectra $B_{l_1 l_2 l_3} = \bigl( ( F_{l_1}, F_{l_2} )_{l_1 + l_2 - l_3} , F_{l_3} \bigr)_{2 l_3}$ for $( l_1, l_2, l_3 ) \in \cA (L)$.
\item For $L \geq 3$, the invariants $J_4$, $J_6$, $J_{10}$ of the sextic $F_3$. They are obtained from its Igusa--Clebsch invariants by \cite[Eqs.~(20)--(21)]{sh-181}, $J_2 = I_2 / 8$, $J_4 = ( 4 J_2^2 - I_4 ) / 96$, $J_6 = ( 8 J_2^3 - 160 J_2 J_4 - I_6 ) / 576$, and $J_{10} = I_{10} / 4096$.
\end{enumerate}
The number of components is $N (2) = 4$ and $N (L) = L + \abs{\cA (L)} + 3$ for $L \geq 3$. The odd bispectra $B_{l_1 l_2 l_3}$ with $( l_1, l_2, l_3 ) \in \cA^- (L)$ are given by the same formula. In the symmetric $3j$ normalization, an invariant trilinear contraction changes by $(-1)^{l_1+l_2+l_3}$ under interchange of two factors; another nonzero normalization gives the same vanishing criterion. Thus an odd bispectrum with two equal bands vanishes when evaluated on the same band vector twice; equivalently, after reordering the contraction, $(F,F)_k=0$ for odd $k$. The pool $\Phi^+_L$ consists of $\Phi_L$ and the odd bispectra divided by $i$, and it has $N^+ (L) = N (L) + \abs{\cA^- (L)}$ components. We call $\Phi_L$ the \textbf{spectral pool}, or the \textbf{even pool} when it is contrasted with $\Phi^+_L$, and we call $\Phi^+_L$ the \textbf{augmented pool}.

The powers are the power spectrum of \cite{KazhdanFunkhouserRusinkiewicz03} and the $B_{l_1 l_2 l_3}$ are bispectra in the sense of \cite{Kakarala12}. The pool is the radial invariant block of \cite[Sec.~5.2]{sh-182}. There the power of band $l$ is $\norm{f_l}^2$. Its complexification is a nonzero multiple of $P_l$, because an invariant quadratic form on the irreducible module of binary forms of degree $2l$ is unique up to a scalar. Let $d_l$ be the degree of a component in $F_l$. For every component of $\Phi_L$ the sum $\sum_l l \, d_l$ is even, and for every odd bispectrum it is odd. So the components of $\Phi_L$ are real on real harmonics and the odd bispectra are purely imaginary by \cite[Lem.~3.6(iii)]{sh-181}. Hence $\Phi_L$ and $\Phi^+_L$ are descriptors on $U_L$. The block of \cite{sh-182} does not contain the odd bispectra.

\begin{table}[ht]
\centering
\setlength{\tabcolsep}{4pt}
\begin{tabular}{rrrrrrrrr}
\toprule
$L$ & $n$ & $n - 3$ & $N (L)$ & $r (L)$ & $n - 3 - r (L)$ & $N^+ (L)$ & $r^+ (L)$ & $n - 3 - r^+ (L)$ \\
\midrule
2 & 8 & 5 & 4 & 4 & 1 & 4 & 4 & 1 \\
3 & 15 & 12 & 10 & 10 & 2 & 10 & 10 & 2 \\
4 & 24 & 21 & 16 & 16 & 5 & 17 & 17 & 4 \\
5 & 35 & 32 & 22 & 22 & 10 & 24 & 24 & 8 \\
6 & 48 & 45 & 32 & 32 & 13 & 37 & 37 & 8 \\
7 & 63 & 60 & 42 & 42 & 18 & 50 & 50 & 10 \\
8 & 80 & 77 & 57 & 57 & 20 & 71 & 71 & 6 \\
9 & 99 & 96 & 72 & 72 & 24 & 92 & 92 & 4 \\
10 & 120 & 117 & 93 & 93 & 24 & 123 & 117 & 0 \\
11 & 143 & 140 & 114 & 114 & 26 & 154 & 140 & 0 \\
12 & 168 & 165 & 142 & 142 & 23 & 197 & 165 & 0 \\
13 & 195 & 192 & 170 & 170 & 22 & 240 & 192 & 0 \\
14 & 224 & 221 & 206 & 206 & 15 & 297 & 221 & 0 \\
15 & 255 & 252 & 242 & 242 & 10 & 354 & 252 & 0 \\
16 & 288 & 285 & 287 & 285 & 0 & 427 & 285 & 0 \\
\bottomrule
\end{tabular}
\caption{Sizes, generic ranks, and generic fiber dimensions of the pools $\Phi_L$ and $\Phi^+_L$.}
\label{tab-pool}
\end{table}

\begin{thm}\label[thm]{thm-pool}
Let $2 \leq L \leq 16$, and let $r (L)$ and $r^+ (L)$ be the generic ranks of $\Phi_L$ and $\Phi^+_L$. Then $r (L) = N (L) < n - 3$ for $L \leq 15$, and $r (16) = n - 3$. Moreover $r^+ (L) = N^+ (L) < n - 3$ for $L \leq 9$, and $r^+ (L) = n - 3$ for $10 \leq L \leq 16$. Let $\Phi$ be $\Phi_L$ or $\Phi^+_L$, and let $U'$ and $M_c$ be as in \cref{thm-fiber}. Then the following hold.
\begin{enumerate}
\item If the generic rank of $\Phi$ is less than $n - 3$ and $c \in U'$, then $M_c / \SO(3)$ is a smooth manifold of positive dimension, listed in \cref{tab-pool}, and $r_{\Phi} ( c ; K ) > 0$ for every compact neighborhood $K$ of $c$.
\item If the generic rank of $\Phi$ equals $n - 3$ and $c \in U'$, then $M_c$ is a finite union of $\SO(3)$-orbits.
\end{enumerate}
\end{thm}

\begin{proof}
By \cref{thm-fiber}, the generic rank is at most the number of components and at most $n - 3$. The counts in \cref{tab-pool} give $N (L) < n - 3$ for $L \leq 15$ and $N (16) > n - 3$, and $N^+ (L) < n - 3$ for $L \leq 9$ and $N^+ (L) > n - 3$ for $L \geq 10$. It remains to bound the ranks from below.

The coefficients of the forms $F_l$ are complex linear coordinates on $U_L \otimes \C$. In these coordinates the components of $\Phi_L$, and the odd bispectra up to the constant $i$, are polynomials with rational coefficients, and they define the complexifications. A minor of the Jacobian that vanishes on $U_L$ vanishes on $U_L \otimes \C$, because a polynomial that vanishes on a real form of a complex vector space is zero. So each generic rank equals that of the complexification. It is at least the rank of the Jacobian at any point of $U_L \otimes \C$.

Put $i = ( F_3, F_3 )_4$ and $\Delta = ( i, i )_2$. In this paragraph $i$ and $\Delta$ denote these covariants, and not the imaginary unit and the Laplacian. In the rows of the Jacobian we replace $J_4$, $J_6$, $J_{10}$ by the invariants $( i, i )_4$, $( i, \Delta )_4$, and the discriminant of $F_3$, of degrees $4$, $6$, $10$. Together with $P_3$, these generate the same algebra as $P_3$, $J_4$, $J_6$, $J_{10}$. Indeed, the Clebsch invariants $( F_3, F_3 )_6$, $( i, i )_4$, $( i, \Delta )_4$ and the Clebsch invariant of degree ten generate the algebra of invariants of even degree of the sextic. The invariants $I_2$, $I_4$, $I_6$, $I_{10}$ are polynomials in them, and $I_{10}$ is the discriminant up to a nonzero factor, with a nonzero coefficient on the degree-ten generator. Each set is therefore a polynomial function of the other. By the chain rule, each Jacobian row space is contained in the other at every point; hence the two row spaces coincide.

We evaluate the Jacobian at a point whose coordinates in the monomial bases are integers in $[-9, 9]$. The coordinates are produced, band by band and coefficient by coefficient, by the linear congruential generator $x \mapsto ( 1664525 \, x + 1013904223 ) \bmod 2^{32}$, started at $x = 1$, with coordinate $( \lfloor x / 2^{16} \rfloor \bmod 19 ) - 9$. The same point serves for every $L$. The entries of the Jacobian are rational numbers. Their denominators divide products of factorials of integers at most $2L \leq 32$. So they are prime to the prime $p = 2^{31} - 1$. The rank of the reduction modulo $p$ is at most the rank over $\Q$. At this point the reductions have the ranks listed in \cref{tab-pool}. This gives the lower bounds. Parts (1) and (2) follow from \cref{thm-fiber}.
\end{proof}

The accompanying program \texttt{rank\_witness.py} performs this computation exactly modulo $p$ for $2 \leq L \leq 16$. For each $L$ and each pool it prints the evaluation point, the labels of a maximal linearly independent set of Jacobian rows, the pivot columns of that row set, and the determinant of the resulting square minor modulo $p$. A nonzero determinant certifies the rank. For $L \leq 15$ the minor of $\Phi_L$ has size $N (L)$, and for $L = 16$ it has size $285$. The starting value of the generator is a parameter of the program; the starting values $2$, $3$ and $7$ were also tested and produced the same ranks. At the cutoffs $L\in\{4,6,8\}$ of \cite[Sec.~7.3]{sh-182}, the fiber dimensions are $5,13,20$ for $\Phi_L$ and $4,8,6$ for $\Phi_L^+$.

We compare these ranks with known orbit-recovery results. For one spherical function including the constant band, \cite[Thm.~4.1 and Cor.~4.2]{FanEtAl24} proves finite generic recovery from invariants of degree at most three for every $L\geq10$. On $\Harm_0 \oplus U_L$ the invariants of degree at most three are polynomials in $f_0$, the powers, and the even and odd bispectra. Finite generic recovery gives their map the generic rank $( L + 1 )^2 - 3$. Fixing $f_0$ lowers the rank by at most one, so the powers and bispectra have generic rank at least $n-3$ on $U_L$; adjoining the sextic invariants cannot lower it. In particular, $r^+(L)=n-3$ holds for every $L\geq10$, and the saturation of $\Phi^+_L$ at $L = 10$ is not a new generic-recovery threshold. The comparison made here concerns the even pool $\Phi_L$, its radii and mirror ambiguity, and the information retained for particular tasks.

For $R$ spherical shells, \cite[Cor.~5.1]{EdidinSatriano24} gives recovery under $L\geq3$ and $R\geq L+2$, while \cite[Thm.~2.2]{BendoryEtAl25} gives generic separation from degree-at-most-three invariants for $R\geq3$. The radial function considered here has one copy of each harmonic band, whereas those results require several independently sampled shells.

The odd bispectra change sign under reflection, and $\Phi_L$ does not change. Let $\tau \colon U_L \to U_L$ be $\tau ( f_1, \dots, f_L ) = ( - f_1, f_2, \dots, (-1)^L f_L )$, so that $\tau$ replaces the radial function by $p \mapsto \rho (-p)$. Put $\chi (c) = \det ( a, Q a, Q^2 a )$, as in \cref{sec-fibers}.

\begin{prop}\label[prop]{prop-mirror}
Let $L \geq 2$. Then the following hold.
\begin{enumerate}
\item The map $\tau$ commutes with the action of $\SO(3)$. Moreover $\Phi_L \circ \tau = \Phi_L$, $B \circ \tau = - B$ for every odd bispectrum $B$, and $\chi \circ \tau = - \chi$.
\item For every $c \in \Omega_L$, the point $\tau c$ lies in the fiber of $\Phi_L$ through $c$ and not in $\SO(3) \cdot c$. In particular $\Phi_L$ does not separate orbits.
\item Let $U'$ and $M_c$ be as in \cref{thm-fiber} for $\Phi = \Phi_L$. Then $\tau (U') = U'$ and $\tau (M_c) = M_c$ for every $c \in U'$. For $L = 16$ and $c \in U'$, the number of orbits in $M_c$ is even, and at least two.
\end{enumerate}
\end{prop}

\begin{proof}
(1) For $R \in \SO(3)$, $( R \cdot \tau f ) (p) = f ( - R^{-1} p ) = ( \tau ( R \cdot f ) ) (p)$. The map $\Psi_l$ is linear, so $\tau$ multiplies $F_l$ by $(-1)^l$. A component of degree $d_l$ in $F_l$ is multiplied by $(-1)^{\sum_l l \, d_l}$, which is $1$ for the components of $\Phi_L$ and $-1$ for the odd bispectra. The map $\tau$ fixes $Q$ and negates $a$, so $\chi ( \tau c ) = \det ( - a, - Q a, - Q^2 a ) = - \chi (c)$.

(2) By (1), $\Phi_L ( \tau c ) = \Phi_L (c)$. The function $\chi$ is constant on orbits by the proof of \cref{lem-free}, and $\chi ( \tau c ) = - \chi (c) \neq \chi (c)$ for $c \in \Omega_L$. Hence $\tau c \notin \SO(3) \cdot c$. The set $\Omega_L$ is not empty, so $\Phi_L$ does not separate orbits.

(3) By (1), $\tau$ is linear with $\Phi_L \circ \tau = \Phi_L$, so $d ( \Phi_L )_{\tau c} \circ \tau = d ( \Phi_L )_c$, and the rank is the same at $c$ and $\tau c$. By (1) also $\tau ( \Omega_L ) = \Omega_L$. Hence $\tau ( U' ) = U'$ and $\tau ( M_c ) = M_c$. By (1), $\tau$ permutes the orbits in $M_c$, and $\tau^2 = 1$. No orbit is fixed, because $\chi$ is constant and nonzero on each orbit in $M_c$ and changes sign under $\tau$. For $L = 16$ the orbits in $M_c$ are finite in number by \cref{thm-pool}, so their number is even. It is at least two, since $M_c$ contains $c$ and $\tau c$.
\end{proof}

The point $\tau c$ matches every component of $\Phi_L$ exactly. It is a distinct mirror decoy precisely when $\tau c\notin\SO(3)\cdot c$, equivalently when the surface has no orientation-reversing orthogonal symmetry about the chosen origin. Such a symmetry need not be reflection in a plane. On $\Omega_L$ the two rotation classes are always distinct. At $L=16$, continuous ambiguity disappears on the free maximal-rank locus, but the mirror pair remains. The proposition neither counts all other finite ambiguities nor asserts that singular parts of a fiber have been classified.

At $L = 2$ the pool consists, up to nonzero factors, of $\abs{a}^2$, $\tr Q^2$, $a^\top Q a$, and $\det Q$. Its fibers have dimension one. By \cite[Thm.~4.9]{sh-181}, adjoining $a^\top Q^2 a$ and $\chi = \det ( a, Q a, Q^2 a )$ gives a separating descriptor. These two invariants have degrees four and six, and every bispectrum is cubic. The second is odd under $\tau$.

For the two adversaries of \cref{sec-radii} these results read as follows. For $2 \leq L \leq 15$ and every $c \in U'$, the decoy radius $r^{\varepsilon}_{\Phi_L} ( c ; K )$ is positive for every tolerance $\varepsilon$ and every compact neighborhood $K$ of $c$. A recognizer built on $\Phi_L$ accepts, near every generic enrolled surface, a family of decoys of the dimension listed in \cref{tab-pool}, and it accepts $\tau c$ whenever $\tau c \notin \SO(3) \cdot c$, at every $L$. For $2 \leq L \leq 15$, a captured value of $\Phi_L$ at a point of $U'$ is consistent with a family of rotation classes of positive dimension, and that family contains the mirror class. For $L = 16$, within $U'$, a captured value is consistent with finitely many rotation classes, including at least the mirror pair; additional finite ambiguities are not excluded. For $L\geq4$ the odd bispectrum $B_{234}$ is a nonzero trilinear invariant, so it is nonzero on an open dense subset and distinguishes $c$ from $\tau c$ there. There are no odd bispectra in these pools for $L=2,3$. Neither the addition of odd bispectra nor full generic rank proves global separation of all orbits.

\section{An explicit finite-tolerance benchmark}\label{sec-decoy}

In this section $w_l = 1$ for all $l$, so $\norm{\cdot}$ is the $L^2$ norm for the round probability measure. For $\theta \in \R^3$ let $R (\theta) \in \SO(3)$ be the rotation with rotation vector $\theta$. The closed ball of radius $\pi$ is mapped onto $\SO(3)$.

\begin{lem}\label[lem]{lem-lipschitz}
Let $c, c' \in U_L$ with $c' = ( f_1', \dots, f_L' )$, and put
\[
K' = \Bigl( \sum_{l=1}^{L} l^2 \norm{f_l'}^2 \Bigr)^{1/2} .
\]
Then the following hold.
\begin{enumerate}
\item $\abs{ \norm{c - R (\theta_1) \cdot c'} - \norm{c - R (\theta_2) \cdot c'} } \leq K' \abs{\theta_1 - \theta_2}$ for all $\theta_1, \theta_2 \in \R^3$.
\item For $R_1, R_2 \in \SO(3)$ let $\phi ( R_1, R_2 ) \in [0, \pi]$ be the rotation angle of $R_1^{-1} R_2$. Then $\abs{ \norm{c - R_1 \cdot c'} - \norm{c - R_2 \cdot c'} } \leq K' \phi ( R_1, R_2 )$.
\end{enumerate}
\end{lem}

\begin{proof}
(2) Write $R_1^{-1} R_2 = R ( \phi w )$ with a unit vector $w$ and $\phi = \phi ( R_1, R_2 )$. On $\Harm_l$ the generator of the rotations about a unit axis is skew, with eigenvalues $i m$ for $\abs{m} \leq l$. So the curve $t \mapsto R ( t \phi w ) \cdot c'$ has speed at most $\phi K'$, and $\norm{R_1 \cdot c' - R_2 \cdot c'} = \norm{c' - R_1^{-1} R_2 \cdot c'} \leq K' \phi$. The triangle inequality completes the proof.

(1) Put $\theta (t) = \theta_1 + t ( \theta_2 - \theta_1 )$ and $R_t = R ( \theta (t) )$ for $t \in [0, 1]$. Then $\dot R_t = R_t \, \widehat{\omega (t)}$, where $\widehat{\omega}$ is the skew matrix of $\omega \in \R^3$. The derivative of the exponential map of $\SO(3)$ at $\theta$ is $\exp ( \widehat{\theta} )$ composed with a normal operator on $\R^3$ whose eigenvalues are $1$ and $( 1 - e^{\mp i \abs{\theta}} ) / ( \pm i \abs{\theta} )$. These have modulus at most $1$, so $\abs{\omega (t)} \leq \abs{\theta_2 - \theta_1}$. On $\Harm_l$ the generator of the rotations about a unit axis is skew, with eigenvalues $i m$ for $\abs{m} \leq l$. Hence
\[
\Bigl\lVert \frac{d}{dt} \bigl( R_t \cdot c' \bigr) \Bigr\rVert \leq \abs{\omega (t)} \Bigl( \sum_{l=1}^{L} l^2 \norm{f_l'}^2 \Bigr)^{1/2} \leq K' \abs{\theta_2 - \theta_1} ,
\]
because $R_t$ acts isometrically. Integration gives $\norm{R (\theta_1) \cdot c' - R (\theta_2) \cdot c'} \leq K' \abs{\theta_1 - \theta_2}$, and the triangle inequality completes the proof.
\end{proof}

A second parametrization avoids trigonometric functions. For a unit quaternion $q = ( q_0, q_1, q_2, q_3 )$ let $R (q) \in \SO(3)$ be the rotation it represents; the entries of $R (q)$ are quadratic forms in $q$, and $R (q) = R (-q)$. For $u \in \R^3$ and $k \in \{ 0, 1, 2, 3 \}$ let $S_k (u) = R (q)$, where $q$ is the unit quaternion whose $k$-th coordinate is $( 1 + \abs{u}^2 )^{-1/2}$ and whose remaining three coordinates, in order, are the coordinates of $( 1 + \abs{u}^2 )^{-1/2} u$. The entries of $S_k (u)$ are rational functions of $u$ with denominator $1 + \abs{u}^2$.

\begin{lem}\label[lem]{lem-charts}
Every rotation is $S_k (u)$ for some $k$ and some $u \in [-1, 1]^3$. For $u, u' \in [-1, 1]^3$ with $\abs{u - u'} < \sqrt{2}$,
\[
\phi \bigl( S_k (u), S_k (u') \bigr) \leq \frac{\pi}{\sqrt{2}} \, \abs{u - u'} .
\]
\end{lem}

\begin{proof}
Let $q_k$ be a coordinate of largest absolute value of a unit quaternion $q$; then $q_k \neq 0$. Replacing $q$ by $-q$, we may take $q_k > 0$; then $q = q_k ( 1, u )$ up to the position of the coordinate $1$, with $\abs{u_j} = \abs{q_j} / q_k \leq 1$. This proves the first claim. Let $q$ and $q'$ be the unit quaternions of $S_k (u)$ and $S_k (u')$. The map $v \mapsto v / \abs{v}$ has differential of norm at most $1 / \abs{v}$, so on the segment from $( 1, u )$ to $( 1, u' )$, where $\abs{v} \geq 1$, it is $1$-Lipschitz; hence $\abs{q - q'} \leq \abs{u - u'} < \sqrt{2}$. Put $\psi = \arccos \ip{q}{q'}$. From $\abs{q - q'}^2 = 2 - 2 \ip{q}{q'} < 2$ we get $\ip{q}{q'} > 0$ and $\psi \in [0, \pi/2)$. The rotation $R (q)^{-1} R (q')$ is represented by the quaternion $\bar q q'$, whose real part is $\ip{q}{q'} = \cos \psi$, so its rotation angle is $2 \psi$. Since $\sin$ is concave on $[0, \pi/2]$, $\sin ( \psi / 2 ) \geq ( 2 \sqrt{2} / \pi ) ( \psi / 2 )$ for $\psi / 2 \in [0, \pi/4]$. Hence $\abs{q - q'} = 2 \sin ( \psi / 2 ) \geq ( \sqrt{2} / \pi ) ( 2 \psi )$, which is the claim.
\end{proof}

\begin{cor}[Validated rotation covering]\label[cor]{cor-validated-rotation}
Let $c, c' \in U_L$, let $K'$ be as in \cref{lem-lipschitz}, and let $\overline K \geq K'$. Then the following hold.
\begin{enumerate}
\item Let $\mathcal B$ be a finite collection of boxes covering $[-\pi,\pi]^3$. For $B\in\mathcal B$ let $\theta_B$ be its center, let $h_B$ bound the distance from the center to every point of the box, and let $\ell_B \leq \norm{c - R ( \theta_B ) \cdot c'}$ be a validated lower bound. Then
\[
\delta(c,c')\geq
\max\bigl\{0,\min_{B\in\mathcal B}(\ell_B-\overline K h_B)\bigr\}.
\]
\item For $k = 0, 1, 2, 3$ let $\mathcal B_k$ be a finite collection of boxes covering $[-1, 1]^3$. For $B \in \mathcal B_k$ let $u_B$ be its center, let $h_B < \sqrt{2}$ bound the distance from the center to every point of the box, and let $\ell_B \leq \norm{c - S_k ( u_B ) \cdot c'}$ be a validated lower bound. Then
\[
\delta(c,c')\geq
\max\Bigl\{0,\min_{k}\min_{B\in\mathcal B_k}\Bigl(\ell_B-\frac{\pi}{\sqrt{2}}\,\overline K h_B\Bigr)\Bigr\}.
\]
\end{enumerate}
Any validated upper bound on $\norm{c-R_0\cdot c'}$ at a single rotation $R_0$ is an upper bound on $\delta(c,c')$.
\end{cor}

\begin{proof}
(1) For $\theta\in B$, \cref{lem-lipschitz}(1) gives $\norm{c-R(\theta)\cdot c'}\geq\ell_B-\overline K h_B$. The boxes cover all rotation vectors of length at most $\pi$, which represent every rotation. Taking the minimum over all boxes gives the stated bound. (2) For $u \in B \in \mathcal B_k$, \cref{lem-lipschitz}(2) and \cref{lem-charts} give $\norm{c - S_k (u) \cdot c'} \geq \ell_B - ( \pi / \sqrt{2} ) \overline K h_B$, and by \cref{lem-charts} the four families cover $\SO(3)$. The upper bound follows by evaluating the function at one rotation.
\end{proof}

For the computation we parametrize the null cone of $\C^3$ by $( x, y, z ) = ( z_0^2 - z_1^2 , \ i ( z_0^2 + z_1^2 ) , \ - 2 z_0 z_1 )$ and take $F_l$ to be the restriction of $f_l$. Two parametrizations differ by a linear substitution in $( z_0, z_1 )$. Under it every component of $\Phi_L$ is multiplied by a nonzero constant, because transvectants are relative invariants. Relative differences of components do not depend on the choice. We fix bases of $\Harm_1, \dots, \Harm_6$ consisting of harmonic polynomials with integer coefficients, listed in the data file \texttt{sh-183-benchmark.json} that accompanies this paper together with the programs used below. The enrolled tuple is
\[
\begin{split}
f_1 & = ( -2, -3, 0 ) , \qquad f_2 = ( -3, 3, 3, 3, 2 ) , \qquad f_3 = ( -1, -3, 3, -4, 2, 2, -4 ) , \\
f_4 & = ( 3, 0, -1, -3, 1, -4, -4, -4, 4 ) , \qquad f_5 = ( -4, 2, -1, 2, -4, 4, -1, 3, 3, 4, -1 ) , \\
f_6 & = ( 1, -1, -1, 3, 0, -4, 2, 4, -3, -2, 0, -3, 1 ) ,
\end{split}
\]
in coordinates with respect to these bases, and $\norm{c}^2 = 1633361 / 45045$, so $\norm{c} < 6.0217$. The decoy $c'$ has coordinates in $10^{-20} \Z$; they are listed in the data file.

\begin{exa}[Certified $L=6$ benchmark]\label[exa]{thm-decoy}
The tuples $c, c' \in U_6$ above satisfy the following.
\begin{enumerate}
\item For every component $\varphi$ of $\Phi_6$, $\abs{\varphi (c') - \varphi (c)} \leq 2 \cdot 10^{-18} \abs{\varphi (c)}$. In particular $\abs{\varphi (c') - \varphi (c)} \leq 10^{-14} \abs{\varphi (c)}$.
\item $21/4 \leq \delta ( c, c' ) \leq 5.3643$. In particular $\delta ( c, c' ) \geq 0.87 \, \norm{c}$.
\item $4.5 \leq \delta ( c, \tau c ) \leq 4.8048$ and $4.3 \leq \delta ( \tau c, c' ) \leq 4.6035$.
\item Let $\varepsilon = 10^{-14} \abs{\Phi_6 (c)}$, and let $K \subset U_6$ be compact and contain $c$ and $c'$. Then $r^{\varepsilon}_{\Phi_6} ( c ; K ) \geq 21/4$.
\item $s_c > -20$ and $s_{c'} > -20$ on $\Sph$, where $s_c = f_1 + \dots + f_6$ for $c = ( f_1, \dots, f_6 )$ and $s_{c'} = f_1' + \dots + f_6'$ for $c' = ( f_1', \dots, f_6' )$.
\end{enumerate}
\end{exa}

\noindent\textbf{Computation.}
(1) The components are evaluated at $c$ and $c'$ in exact arithmetic over $\Q (i)$, with the invariants $J_4, J_6, J_{10}$ computed from the Clebsch covariants of $F_3$. All $32$ values are real and nonzero, and the largest relative difference is $1.54 \cdot 10^{-18}$.

(2) The norms are computed from the exact Gram matrices of the bases. The upper bound is the exact value of $\norm{c - R_0 \cdot c'}$ at the rational rotation $R_0 = R (q_0)$ of the quaternion $q_0 = ( 1427 / 905074 , \ 290689 / 653469 , \ -292467 / 330422 , \ 1 )$, normalized; $\norm{c - R_0 \cdot c'}^2$ is a rational number smaller than $28.77556 < 5.3643^2$. The lower bound is \cref{cor-validated-rotation}(2). Each cube $[-1, 1]^3$ is divided into $8^3$ boxes. For a box $B$ the value $\norm{c - S_k ( u_B ) \cdot c'}^2 = \norm{c}^2 + \norm{c'}^2 - 2 \ip{c}{S_k ( u_B ) \cdot c'}$ is enclosed by evaluating $\ip{c}{S_k ( u_B ) \cdot c'}$ as an explicit polynomial in the nine entries of $S_k ( u_B )$, in interval arithmetic, in IEEE double precision with rounding to nearest and with outward rounding by one unit in the last place after every operation. With $\overline K = 22.806 > K'$ and $h_B$ the half-diagonal of $B$, a box with $\ell_B - ( \pi / \sqrt{2} ) \overline K h_B \geq 21/4$ is discarded, and every other box is divided into eight. All boxes are discarded after at most eight levels, with $990\,624$ center evaluations in total. The ratio follows from $( 21/4 )^2 > 0.87^2 \cdot 1633361 / 45045$.

(3) The same procedure, applied to the pairs $( c, \tau c )$ and $( \tau c, c' )$ with the thresholds $9/2$ and $43/10$, discards all boxes after $249\,744$ and $221\,200$ evaluations. The upper bounds come from the rational quaternions listed in the data file.

(4) By (1), $\abs{\Phi_6 (c') - \Phi_6 (c)} \leq \varepsilon$. By (2), $c' \notin \SO(3) \cdot c$ and $\delta ( c, c' ) \geq 21/4$.

(5) Write $p = v / \abs{v}$ with $v$ on a face of the cube $[-1, 1]^3$, so that $s_c (p) = \sum_l f_l (v) / \abs{v}^l$. Each face is covered by boxes, and $s_c$ is enclosed on each box by interval arithmetic as in (2), with boxes subdivided until the lower bound exceeds $-20$.

The distance bounds distinguish $c'$ from both $c$ and $\tau c$. The mirror match itself is an exact algebraic identity by \cref{prop-mirror}. The rational inequality in part (1) is a finite-tolerance match: it does not establish that $c'$ belongs to the exact fiber or to a particular connected component of that fiber.

We found $c'$ by numerical continuation along the descriptor level set, moving in the quotient tangent space of \cref{thm-pool}, which has dimension $13$ at a free maximal-rank point, with Gauss--Newton correction of the descriptor residual, followed by Newton iteration in $60$-digit arithmetic and rounding of the coordinates to $10^{-20} \Z$. The continuation is not part of the certificate; parts (1)--(5) are verified from the final rational pair alone.

Part (1) is a componentwise relative check. No component of $\Phi_6 (c)$ vanishes, so it is the discrepancy of \cref{sec-radii} with $W$ the diagonal matrix of the reciprocals $1 / \abs{\varphi (c)}$ and with the maximum norm in place of the Euclidean norm. A matcher on $\Phi_6$ with componentwise relative tolerance $10^{-14}$ therefore accepts $c'$, and the Euclidean tolerance $\varepsilon$ of part (4) is a consequence. The band powers of $c$ and $c'$ agree to the same precision, so their energies agree to that precision in every band. By (5), the radial functions $24 + s_c$ and $24 + s_{c'}$ exceed $4$ on the sphere. \Cref{fig-decoy} shows both as radial surfaces with constant band $f_0 = 24$.

\begin{rem}[Certification status]\label[rem]{rem-numerical-status}
Part (1) is established in exact rational arithmetic. Parts (2), (3) and (5) are established by interval arithmetic in IEEE double precision in round-to-nearest mode, with every elementary result widened by one unit in the last place in each direction; this covers the rounding error of correctly rounded addition, subtraction, multiplication, division and square root. The constant $\pi / \sqrt{2}$ is enclosed from above by the double-precision quotient of the nearest doubles to $\pi$ and $\sqrt{2}$, increased by two units in the last place; this exceeds $\pi / \sqrt{2}$ because the nearest double to $\pi$ lies below $\pi$ by less than $1.3 \cdot 10^{-16}$ and the nearest double to $\sqrt{2}$ lies above $\sqrt{2}$ by less than $10^{-16}$. The entries of $S_k ( u_B )$ have absolute value at most $1$ and all coefficients are bounded, so no overflow occurs, and results in the subnormal range remain correctly rounded. The certificate depends on the correctness of this floating-point arithmetic and of the interval routines in the accompanying program \texttt{certify.py}, and on nothing else. The exact value of $\norm{c - R_0 \cdot c'}^2$, the rational quaternions, the box counts and the Lipschitz constants are recorded by the program.
\end{rem}

The factor $0.87$ is normalized by the nonconstant coefficient tuple. For the rendered model with $\rho_0=24$,
\[
\norm{\rho_c}_{L^2(d\mu)}=\sqrt{576+1633361/45045}\approx24.744.
\]
The distance interval $[5.25,5.3643]$ corresponds to approximately $0.2122$--$0.2168$ of that full radial norm. Both normalizations are radial $L^2$ quantities. They are neither Hausdorff distances nor violations of a dimensional tolerance.

\begin{figure}[ht]
\centering
\includegraphics[width=0.9\textwidth]{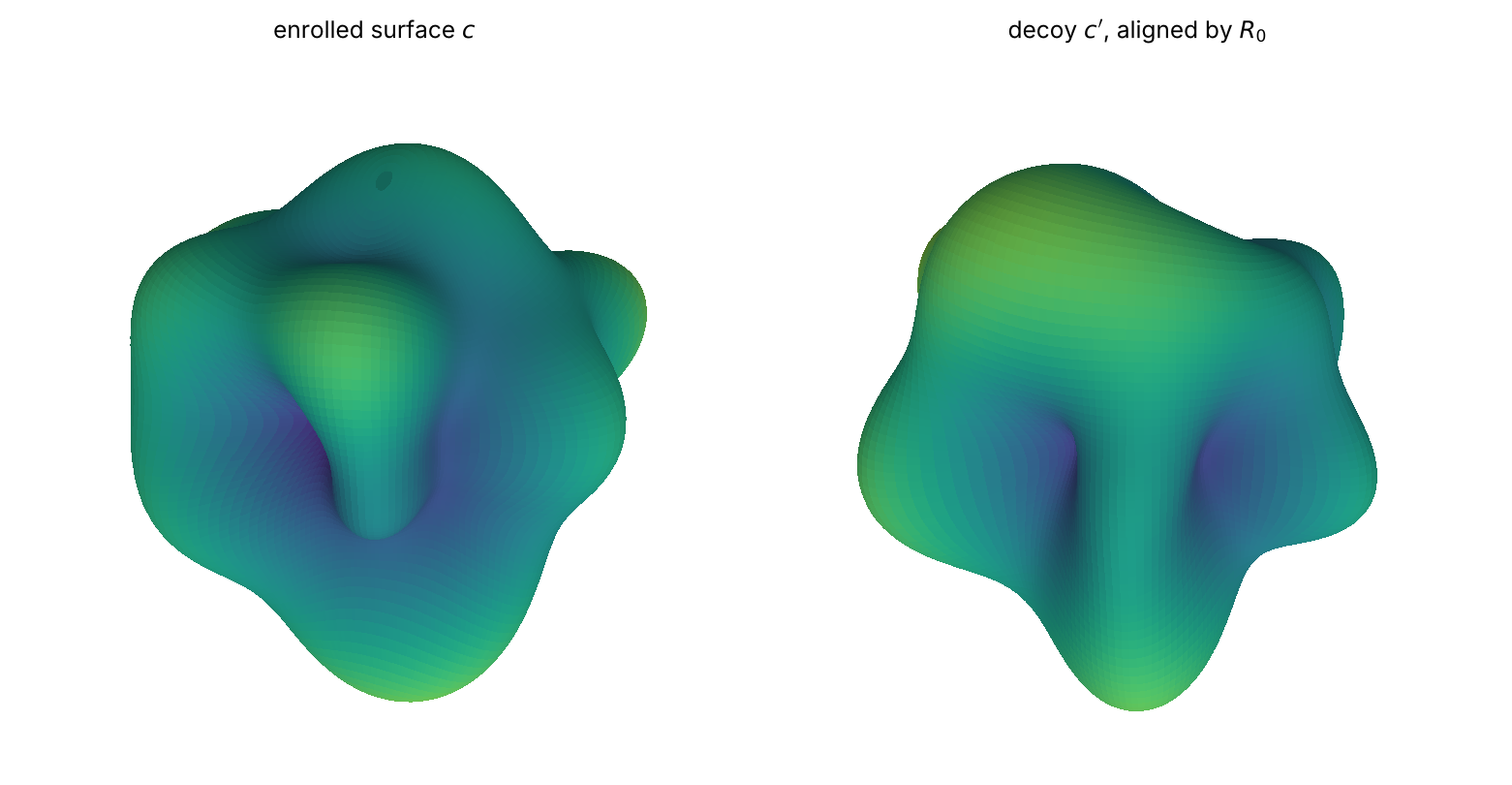}
\caption{The enrolled surface $c$ (left) and the decoy $c'$ (right), drawn with $f_0 = 24$; $c'$ is rotated by the rotation $R_0$ of \cref{thm-decoy}(2). All $32$ components of $\Phi_6$ agree within relative tolerance $2 \cdot 10^{-18}$; see \cref{rem-numerical-status}.}
\label{fig-decoy}
\end{figure}

\FloatBarrier
\section{Task-specific certification and disclosure}\label{sec-tasks}

A shape collision is not necessarily an error for every task. An inspection quantity can be constant on a fiber that contains many noncongruent shapes. Conversely, a shape descriptor can fail to determine geometry while revealing an attribute exactly. The distinction can be stated without choosing a classifier.

\begin{defn}\label[defn]{def-task-interval}
Let $T\colon U_L\to\R$ be a continuous rotation-invariant task quantity. For a nonempty set $\mathcal F=\mathcal F^{\varepsilon}_{\Phi,W}(y;K)$, define
\[
\begin{split}
T_-(y) & =\min_{x\in\mathcal F}T(x),\qquad
T_+(y)=\max_{x\in\mathcal F}T(x),\\
A^{\varepsilon}_{\Phi,W,T}(y;K) & =T_+(y)-T_-(y).
\end{split}
\]
We call $[T_-(y),T_+(y)]$ the \textbf{task interval} and $A^{\varepsilon}_{\Phi,W,T}(y;K)$ its width.
\end{defn}

The interval is an enclosure of the attainable task values; these values need not fill it when the consistent set is disconnected. For a requirement $T\leq b$, an established upper bound $t_{\mathrm{up}}\geq T_+(y)$ with $t_{\mathrm{up}}\leq b$ certifies compliance. An established lower bound $t_{\mathrm{low}}\leq T_-(y)$ with $t_{\mathrm{low}}>b$ certifies noncompliance. To demonstrate ambiguity one needs two feasible shapes with validated task values on opposite sides of $b$. A loose enclosing interval that crosses $b$ does not establish such a pair.

\begin{prop}[Task information]\label[prop]{prop-task-information}
With the notation of \cref{def-task-interval}, the following hold.
\begin{enumerate}
\item The smallest worst-case error of a scalar estimate based on $y$ is
\[
\min_{t\in\R}\max_{x\in\mathcal F}\abs{t-T(x)}
=\frac{T_+(y)-T_-(y)}{2}.
\]
The midpoint of the task interval attains this value.
\item If $\abs{T(x)-T(x')}\leq M_T\delta(x,x')$ for all $x, x' \in K$, then
\[
A^{\varepsilon}_{\Phi,W,T}(y;K)\leq M_TD^{\varepsilon}_{\Phi,W}(y;K).
\]
In particular, for $y=\Phi(c)$,
\[
\max_{x\in\mathcal F}\abs{T(x)-T(c)}\leq M_T r^{\varepsilon}_{W\Phi}(c;K).
\]
\item At exact tolerance, every task interval is a singleton if and only if there is a map $g\colon\Phi(K)\to\R$ such that $T|_K=g\circ\Phi|_K$. In that case $g$ is continuous.
\end{enumerate}
\end{prop}

\begin{proof}
For (1), both endpoints are attained, so every estimate has error at least half their difference. Every attainable value lies between them, so their midpoint has error at most half their difference. For (2), apply the assumed inequality to a maximizing pair and, for the second assertion, to $c$ and a feasible point. For (3), singleton intervals say exactly that $T$ is constant on each fiber in $K$, making $g(\Phi(x))=T(x)$ well-defined. Conversely, such a factorization makes the intervals singletons. The continuous surjection $\Phi|_K$ from a compact space onto the Hausdorff space $\Phi(K)$ is a quotient map, so continuity of $T|_K$ gives continuity of $g$.
\end{proof}

For exact descriptor collisions, every deterministic predictor $h\circ\Phi$ takes the same value at both shapes. Thus if the required task values differ, no such predictor can be correct at both. At positive tolerance, closeness of two descriptors alone does not imply that an arbitrary predictor assigns them the same label. The corresponding impossibility statement concerns a common observation $y$: if both shapes lie in its consistent set, no decision based only on that observation can be guaranteed correct for both under the stated error model.

A local test uses the same Jacobian calculation as the fiber theorem.

\begin{prop}[Differential test for task ambiguity]\label[prop]{prop-task-rank}
Let $L \geq 2$, let $\Phi \colon U_L \to \R^N$ be a descriptor, and let $U'$ and $M_c$ be as in \cref{thm-fiber}. Let $c\in U'$, and let $T \colon U_L \to \R$ be a $C^1$ rotation-invariant function. If
\[
\rank d(\Phi,T)_c>\rank d\Phi_c,
\]
then every neighborhood of $c$ in the regular fiber $M_c$ contains points $c_-$ and $c_+$ with
\[
T(c_-)<T(c)<T(c_+).
\]
The same conclusion holds inside an admissible set $K$ whenever $K$ contains a neighborhood of $c$ in that fiber.
\end{prop}

\begin{proof}
The rank inequality is equivalent to the existence of $v\in\ker d\Phi_c$ with $dT_c(v)\neq0$. By the constant-rank theorem, $\ker d\Phi_c$ is the tangent space to the regular fiber. Choose a $C^1$ curve in that fiber through $c$ with derivative $v$. The derivative of $T$ along the curve is nonzero at $c$, so sufficiently small parameters of opposite signs give the two inequalities.
\end{proof}

The rank test detects first-order variation only. It does not measure the size of the variation, and at a single point it is sufficient but not necessary. It also cannot exclude a jump between disconnected components. For example, let $c$ lie in the free maximal-rank locus of $\Phi_{16}$. The differential of every invariant annihilates the three-dimensional tangent space of the orbit, so adjoining $\chi$ does not increase the rank at $c$. Yet $\chi$ takes opposite signs at $c$ and $\tau c$, which $\Phi_{16}$ does not distinguish. Full local rank therefore does not imply that a task quantity is determined globally.

\section{Volume disclosed by the even spectral pool}\label{sec-volume}

Fix a known constant radial term $\rho_0>0$ and use the round probability measure $d\mu=d\omega/(4\pi)$ on $\Sph$. Put
\[
s_c(p)=\sum_{l=1}^L f_l(p),\qquad \rho_c(p)=\rho_0+s_c(p),
\]
and suppose $\rho_c>0$. The enclosed radial solid is $\{tp:p\in\Sph,\ 0\leq t\leq\rho_c(p)\}$. Its volume is
\begin{equation}\label{eq-radial-volume}
V(c)=\frac{4\pi}{3}\int_{\Sph}\rho_c(p)^3\,d\mu(p).
\end{equation}
For a positive continuous radial function $\rho$ on $\Sph$ we write $V(\rho)=\frac{4\pi}{3}\int_{\Sph}\rho^3\,d\mu$, so that $V(c)=V(\rho_c)$. This is the volume enclosed by the modeled surface. For a hollow object it is neither the material volume nor a measure of mass.

\begin{thm}[Exact volume disclosure]\label[thm]{thm-volume}
For $L\geq2$ and fixed $\rho_0$, there is an explicitly computable vector $\beta(\rho_0)$ such that every $c\in U_L$ with $\rho_c>0$ satisfies
\begin{equation}\label{eq-volume-affine}
V(c)=\frac{4\pi}{3}\rho_0^3+\beta(\rho_0)^\top\Phi_L(c).
\end{equation}
The entries of $\beta$ corresponding to $J_4,J_6,J_{10}$ are zero. The same identity holds for $\Phi_L^+$ with zero coefficients on the odd bispectra. Consequently, at fixed $\rho_0$, every exact fiber intersected with the set $\{ c : \rho_c > 0 \}$ has zero volume-task width, irrespective of its shape dimension or diameter.
\end{thm}

\begin{proof}
Every $f_l$ has mean zero. Orthogonality gives
\[
\int\rho_c^3\,d\mu
=\rho_0^3+3\rho_0\sum_l\norm{f_l}_{L^2(d\mu)}^2
+\sum_{i,j,k=1}^L\int f_if_jf_k\,d\mu.
\]
For each $l$, the quadratic integral is a nonzero fixed multiple $a_lP_l$. A triple integral vanishes when the sum of the degrees is odd, by parity, or when the triangle condition fails. For an ordered even triangle triple, the invariant trilinear space is one-dimensional: in the Clebsch--Gordan decomposition, $\Harm_k$ occurs with multiplicity one in $\Harm_i\otimes\Harm_j$. The triple integral and the iterated transvectant are proportional. The Gaunt formula shows that the former is nonzero for these triples; see \cite[Eq.~(34.3.5) and \S34.3(vii)]{DLMFGaunt}. Thus there are nonzero constants $b_{ijk}$ with $\int f_if_jf_k\,d\mu=b_{ijk}B_{ijk}$. Collecting the ordered triples gives
\begin{equation}\label{eq-volume-expanded}
V(c)=\frac{4\pi}{3}\left(
\rho_0^3+3\rho_0\sum_{l=1}^La_lP_l(c)
+\sum_{(i,j,k)\in\cA(L)}m_{ijk}b_{ijk}B_{ijk}(c)\right),
\end{equation}
where $m_{ijk}=1,3,6$ according as all three indices, exactly two indices, or no indices coincide. This gives the stated vector and shows that the other invariants are unnecessary for volume.
\end{proof}

\subsection{Normalization and exact coefficient computation}\label{sec-volume-coefficients}

The constants must match the descriptor normalization; norm powers cannot be substituted directly for the transvectant powers. For the null-cone parametrization of \cref{sec-decoy}, let $\mathcal P_l$ be the Legendre polynomial with $\mathcal P_l(1)=1$. The solid harmonic extending $\mathcal P_l(p_z)$ restricts to
\[
q_l(z_0,z_1)=(-1)^l\binom{2l}{l}z_0^lz_1^l.
\]
Indeed only the leading term in $\mathcal P_l$ survives restriction to $x^2+y^2+z^2=0$, and $z=-2z_0z_1$. Consequently
\begin{equation}\label{eq-volume-calibration}
\begin{split}
a_l&=\frac{1}{(2l+1)(-1)^l\binom{2l}{l}},\\
b_{ijk}&=\frac{G_{ijk}}{\bigl((q_i,q_j)_{i+j-k},q_k\bigr)_{2k}},
\qquad
G_{ijk}=\frac12\int_{-1}^{1}\mathcal P_i(t)\mathcal P_j(t)\mathcal P_k(t)\,dt.
\end{split}
\end{equation}
The denominator is nonzero for $(i,j,k)\in\cA(L)$. With $s=(i+j+k)/2$, the zonal Gaunt coefficient is
\[
G_{ijk}=
\frac{(s!)^2(2s-2i)!(2s-2j)!(2s-2k)!}
{(s-i)!^2(s-j)!^2(s-k)!^2(2s+1)!},
\]
by \cite[Eq.~(34.3.5) and \S34.3(vii)]{DLMFGaunt}. Thus the calibration uses rational arithmetic only. For example,
\[
\begin{split}
a_1 & =-\frac16,\qquad a_2=\frac1{30},\\
b_{112} & =\frac1{30},\qquad b_{222}=-\frac1{105}.
\end{split}
\]
The accompanying program \texttt{volume\_from\_pool.py} implements these formulas, evaluates volume as an exact rational multiple of $\pi$ when the supplied descriptor is rational, and propagates componentwise absolute error bounds. It does not test realizability of an arbitrary descriptor or validity of a supplied measurement-error bound.

\subsection{Descriptor and radial-model error}\label{sec-volume-error}

\begin{cor}\label[cor]{cor-volume-error}
Let $L \geq 2$, let $c \in U_L$ with $\rho_c > 0$, let $W$ be a nonsingular matrix, and put $\widehat V(y)=\frac{4\pi}{3}\rho_0^3+\beta(\rho_0)^\top y$. If $\abs{W(\Phi_L(c)-y)}\leq\varepsilon$, then
\begin{equation}\label{eq-volume-descriptor-error}
\abs{V(c)-\widehat V(y)}\leq\abs{W^{-\top}\beta(\rho_0)}\,\varepsilon.
\end{equation}
If instead each component satisfies $\abs{\Phi_{L,j}(c)-y_j}\leq e_j$, the right-hand side may be replaced by $\sum_j\abs{\beta_j}e_j$.

Suppose also that the actual radial function $\rho$ and the modeled radial function $\rho_c$ satisfy $0<\rho,\rho_c\leq\rho_{\max}$ and $\norm{\rho-\rho_c}_{L^2(d\mu)}\leq\gamma$. Then
\begin{equation}\label{eq-volume-total-error}
\abs{V(\rho)-\widehat V(y)}
\leq\abs{W^{-\top}\beta(\rho_0)}\,\varepsilon+4\pi\rho_{\max}^2\gamma.
\end{equation}
The componentwise version holds with the same additional term.
\end{cor}

\begin{proof}
Subtract \cref{eq-volume-affine} from the definition of $\widehat V$ and apply Cauchy--Schwarz to 
\[
(W^{-\top}\beta)^\top W(\Phi_L(c)-y).
\]
 The componentwise estimate is the triangle inequality. For $a, b \in [0,\rho_{\max}]$ we have 
 \[
 \abs{a^3-b^3}\leq3\rho_{\max}^2\abs{a-b}.
 \]
  Therefore \cref{eq-radial-volume} and the fact that $\mu$ is a probability measure give
\[
\abs{V(\rho)-V(c)}\leq4\pi\rho_{\max}^2\norm{\rho-\rho_c}_{L^1(d\mu)}
\leq4\pi\rho_{\max}^2\gamma.
\]
Another triangle inequality proves \cref{eq-volume-total-error}.
\end{proof}

An uncertain constant term must be accounted for: its contribution can be included in the radial error $\gamma$ or bounded separately. Similarly, a descriptor of a truncated tuple determines the modeled volume. It determines the physical volume only up to an estimate of the truncation and acquisition errors.

\subsection{An exact clearance test with identical volume}\label{sec-clearance-example}

The following example exhibits a task conflict in closed form, without a numerical search. The reference origin is fixed; no centroid normalization is imposed.

\begin{prop}\label[prop]{prop-clearance-example}
Let $p=(x,y,z)\in\Sph$ and define
\[
\begin{split}
\rho_A(p)&=3+\tfrac23x+\tfrac13y+\tfrac23z+x^2-z^2,\\
\rho_B(p)&=3+y+x^2-z^2.
\end{split}
\]
Both radial functions are at least $1$. Their coefficient tuples $c_A, c_B \in U_2$, taken with $\rho_0 = 3$, have exactly the same even pool $\Phi_2$, and their enclosed volumes are both $216\pi/5$. Nevertheless, for $T_{\mathrm{out}}(c)=\max_{p\in\Sph}\rho_c(p)$,
\[
T_{\mathrm{out}}(c_B)=\frac{17}{4}<\frac92<\frac{14}{3}\leq T_{\mathrm{out}}(c_A).
\]
Thus the same exact descriptor supports incompatible decisions for containment in the origin-centered ball of radius $9/2$, while it determines the enclosed volume exactly.
\end{prop}

\begin{proof}
Both have $Q=\diag(1,0,-1)$, with $a_A=(2/3,1/3,2/3)$ and $a_B=(0,1,0)$. Then
\[
\begin{split}
\abs{a_A}^2 & =\abs{a_B}^2=1,\quad\tr Q^2=2,\quad\det Q=0,\\
a_A^\top Qa_A & =a_B^\top Qa_B=0.
\end{split}
\]
These are the four invariants of $\Phi_2$, up to its fixed nonzero normalization factors. In the null-cone convention of \cref{sec-decoy}, the ordered values $(P_1,P_2,B_{112},B_{222})$ are $(-2,8,0,0)$ for both tuples. Also $\abs{\ip{a}{p}}\leq1$ for $a \in \{ a_A, a_B \}$ and $\abs{p^\top Qp}\leq1$, so each radius is at least $3-1-1=1$. Applying \cref{eq-volume-expanded} with the constants in \cref{eq-volume-calibration} gives
\[
V=\frac{4\pi}{3}\left(27+9\left(\frac13+\frac4{15}\right)\right)=\frac{216\pi}{5}.
\]
For $\rho_B$, use $x^2-z^2\leq1-y^2$ to obtain
\[
\rho_B\leq4+y-y^2\leq\frac{17}{4}.
\]
Equality is attained at $y=1/2$, $z=0$, $x=\sqrt3/2$. For $\rho_A$, evaluation at $p=(1,0,0)$ gives $14/3$. The displayed strict inequalities follow.
\end{proof}

The two objects cannot be congruent even under $\Or(3)$, since $T_{\mathrm{out}}$ is invariant under that group and their values differ. The failure is therefore not solely a handedness ambiguity. At the common descriptor value the task interval for any admissible $K$ containing both tuples has width at least $14/3-17/4=5/12$. By \cref{prop-task-information}, every scalar estimator based on that value has worst-case outer-radius error at least $5/24$ on these two objects. No comparable error is forced for volume, whose exact task width is zero.

\section{Admissible shapes and a representation audit}\label{sec-audit}

\subsection{Origin, scale, and the feasible domain}\label{sec-admissibility}

The set $K$ records geometric side information for both recognition and disclosure. One concrete model, with fixed $\rho_0$, is
\begin{equation}\label{eq-admissible-set}
K=\{c\in U_L:\norm{c}\leq M,\quad
\rho_{\min}\leq\rho_0+s_c(p)\leq\rho_{\max}\text{ for all }p\in\Sph\},
\end{equation}
where $0<\rho_{\min}<\rho_{\max}$. It is compact and rotation-invariant. It is semialgebraic, because each harmonic is the restriction of a polynomial and the quantified conditions on $p\cdot p=1$ can be eliminated over the reals; see \cite{BochnakCosteRoy98}. Positivity of this kind is an assumption on the truncated model: the truncation of a positive function need not be positive.

A declared reference origin and a centroid convention are different assumptions. For the enclosed solid, the volume centroid is at the origin precisely when
\[
\int_{\Sph}p\,\rho_c(p)^4\,d\mu(p)=0,
\]
since its unnormalized first moment is $\frac14\int p\rho_c^4\,d\omega$. A surface-area centroid imposes a different condition. Recentering changes the radial function and generally its harmonic coefficients. The exact example of \cref{prop-clearance-example} uses the declared origin; it is not a counterexample to an unspecified recentered pipeline.

More generally, if equality constraints restrict coefficients to a smooth rotation-invariant manifold $\mathcal M$ of dimension $d$ on which the action is free, the local regular-fiber dimension is
\[
d-3-\rank\bigl(d\Phi_c|_{T_c\mathcal M}\bigr),
\]
where the restricted rank is assumed constant in a neighborhood. This follows from the same constant-rank and quotient-manifold argument as \cref{thm-fiber}. The ambient count $n-3-r$ does not carry over after such a normalization. Physical size must also be retained when the task depends on dimensions; removing $\Harm_0$ for orbit analysis does not make it irrelevant to inspection.

\subsection{A computable geometric bound}\label{sec-geometric-bound}

\begin{prop}\label[prop]{prop-envelope-lipschitz}
For the weighted norm of \cref{sec-radii}, put
\[
\Gamma_{L,w}=\left(\sum_{l=1}^L\frac{2l+1}{w_l}\right)^{1/2}.
\]
Then $\norm{s_c}_{\infty}\leq\Gamma_{L,w}\norm{c}$ for every $c \in U_L$. At a fixed constant term, the outer-radius task satisfies
\[
\abs{T_{\mathrm{out}}(c)-T_{\mathrm{out}}(c')}
\leq\Gamma_{L,w}\delta(c,c')
\]
for all $c, c' \in U_L$. For $w_l=1$, $\Gamma_{L,w}=\sqrt{L^2+2L}$.
\end{prop}

\begin{proof}
Choose an orthonormal basis of each band for $L^2(d\mu)$. The addition theorem makes the sum of its squared values at any point equal to $2l+1$. Weighted Cauchy--Schwarz therefore gives $\abs{s_c(p)}\leq\Gamma_{L,w}\norm{c}$. Apply this to $c-R\cdot c'$ and use invariance of the maximum under rotations. Minimizing over $R$ proves the task bound.
\end{proof}

In particular, $\rho_0-\Gamma_{L,w}M\geq\rho_{\min}$ is a simple sufficient positivity condition on a coefficient ball. It is not necessary. The task bound yields conservative clearance certificates from an upper bound on the shape radius; optimizing the task itself can be substantially tighter.

\subsection{An error budget for the actual representation}\label{sec-error-budget}

A numerical matching tolerance is not a sensor specification. Suppose $c$ and a fitted tuple $\widehat c$ belong to a convex bounded region on which $\norm{W d\Phi_x}_{\mathrm{op}}\leq M_{\Phi,W}$, with the domain norm fixed as in \cref{sec-radii}. If $\norm{\widehat c-c}\leq\sigma$ and the stored value satisfies $\abs{W(y-\Phi(\widehat c))}\leq\eta$, integration along the line segment gives
\[
\abs{W(y-\Phi(c))}\leq M_{\Phi,W}\sigma+\eta.
\]
Thus $\varepsilon=M_{\Phi,W}\sigma+\eta$ is a sufficient descriptor-error budget under these assumptions. Truncation error in a geometric quantity must still be propagated separately, as in \cref{cor-volume-error}. If the acquisition bounds hold only with a stated probability, the resulting certificate holds with that probability and is not unconditional.

\subsection{Use in independent evaluation}\label{sec-evaluation}

The framework is intended primarily for evaluators who select or test geometric representations, not for operators who apply a radius to an unspecified field observation. The input is one documented representation, one admissible shape model, a calibrated error budget, and a task requirement. The output is a task certificate, a verified conflicting pair, or an unresolved case. A failure to find a conflicting shape is not a certificate of its absence. Similarly, a numerical distance computed at one aligning rotation bounds the orbit distance from above only; \cref{cor-validated-rotation} gives the lower bound needed to certify geometric separation.

One such setting is an independent validation study of scan-based identification or inspection of equipment or calibration objects, for instance in defense or aerospace programs. Such a study is relevant only if the representation under test actually discards geometry in the manner modeled here. A comparison against complete geometry does not automatically inherit a collision of $\Phi_L$.

A focused validation would first reproduce the exact quadrupole formulas and \cref{prop-clearance-example}. It would then use a controlled collection of fully observed radial objects with independently measured reference geometry, evaluate band powers, the even pool, and a task-augmented representation under the same acquisition and preprocessing conditions, and report task-error bounds, conflicting decisions, the fraction of unresolved cases, computation time, and descriptor storage. The high-band benchmark of \cref{sec-decoy} is a certified calibration case for such a comparison, and \cref{sec-bennu} applies the same audit to a space object.

For shared records, the same audit specifies what is released: at fixed known $\rho_0$, the modeled enclosed volume is recoverable by \cref{thm-volume}, and multiple records about one tuple must be assessed jointly by \cref{cor-composition}. An application need not demand complete shape reconstruction. It may be better served by retaining one additional task measurement than by enlarging a global invariant pool. That decision depends on the required task and on which attributes may be disclosed.

\section{A case study: asteroid (101955) Bennu}\label{sec-bennu}

Rotation-invariant descriptors are natural for a tumbling object of unknown attitude. We apply the framework to the radar shape model of \cite{NolanEtAl13} for the near-Earth asteroid (101955) Bennu. It is archived in the Planetary Data System as the file \texttt{data/101955bennu.tab} of the bundle
 \texttt{urn:nasa:pds:ast-bennu.radar.shape-model}, version 1.1; see \cite{BennuPDS}. The model is a closed triangulated surface with $1348$ vertices and $2692$ facets. Its origin is the center of mass, and its axes are the principal axes. We use meters.

The polyhedral volume of the model is $V_{\mathrm{mesh}} = 6.22656 \cdot 10^{7} \, \mathrm{m}^3$, and its equivalent-sphere diameter is $491.8 \, \mathrm{m}$. The facets are oriented so that the signed volume is positive. Every facet has outward normal $\nu$ with $\ip{\nu}{a} > 0$ at its vertices $a$; the smallest value of $\ip{\nu}{a} / ( \abs{\nu} \abs{a} )$ is $0.667$. So the origin lies in the open inner half-space of every facet plane, that is, in the interior of the kernel of the polyhedron, and the surface is star-shaped about the origin. As an independent check, the radial projection maps every facet onto a positively oriented spherical triangle, and the solid angles of these triangles sum to $4 \pi$; so the projection has degree one, and a regular direction is covered exactly once. The radial function $\rho$ is positive and Lipschitz. The vertices of the model lie at distances between $224.2 \, \mathrm{m}$ and $289.8 \, \mathrm{m}$ from the center of mass. Since $\abs{x}$ is convex, $\rho \leq 289.8 \, \mathrm{m}$ on $\Sph$.

The constant term $\rho_0$ and the bands $f_1, \dots, f_L$ approximate the orthogonal projections of $\rho$ for the round probability measure; they are computed using a product Gauss rule with $200 \times 400$ nodes, without a validated quadrature bound, since $\rho$ is only piecewise smooth. For $L = 6$ the coordinates of $f_1, \dots, f_6$ in the bases of \cref{sec-decoy} are rounded to $10^{-4} \, \mathrm{m}$, and $\rho_0$ is rounded to $245.2830 \, \mathrm{m}$. The rounded tuple is denoted $c \in U_6$; it is listed in the data file \texttt{sh-183-bennu.json}. All statements about $L = 6$ concern $c$ and its radial model $\rho_c = \rho_0 + s_c$. We have $\norm{c} < 11.52 \, \mathrm{m}$.

\begin{table}[ht]
\centering
\setlength{\tabcolsep}{5pt}
\begin{tabular}{rrrrrr}
\toprule
$L$ & components & $V_{\mathrm{pool}}$ ($\mathrm{km}^3$) & $\abs{V_{\mathrm{pool}} - V_L} / V_{\mathrm{mesh}}$ & $( V_{\mathrm{pool}} - V_{\mathrm{mesh}} ) / V_{\mathrm{mesh}}$ & $\gamma_L$ ($\mathrm{m}$) \\
\midrule
2 & 4 & 0.061969 & $1.1 \cdot 10^{-15}$ & $-0.477\,\%$ & 9.71 \\
3 & 7 & 0.062009 & $8.4 \cdot 10^{-16}$ & $-0.412\,\%$ & 9.01 \\
4 & 13 & 0.062186 & $3.6 \cdot 10^{-16}$ & $-0.127\,\%$ & 4.97 \\
5 & 19 & 0.062204 & $2.4 \cdot 10^{-16}$ & $-0.098\,\%$ & 4.32 \\
6 & 29 & 0.062227 & $< 10^{-16}$ & $-0.063\,\%$ & 3.45 \\
7 & 39 & 0.062232 & $3.6 \cdot 10^{-16}$ & $-0.054\,\%$ & 3.18 \\
8 & 54 & 0.062248 & $3.6 \cdot 10^{-16}$ & $-0.029\,\%$ & 2.31 \\
9 & 69 & 0.062250 & $3.6 \cdot 10^{-16}$ & $-0.026\,\%$ & 2.17 \\
10 & 90 & 0.062253 & $1.2 \cdot 10^{-16}$ & $-0.021\,\%$ & 1.95 \\
11 & 111 & 0.062254 & $3.6 \cdot 10^{-16}$ & $-0.019\,\%$ & 1.88 \\
12 & 139 & 0.062256 & $4.8 \cdot 10^{-16}$ & $-0.016\,\%$ & 1.73 \\
\bottomrule
\end{tabular}
\caption{Volume of Bennu from the even pool. $V_{\mathrm{pool}}$ is the right-hand side of \cref{eq-volume-expanded} evaluated on the powers and even bispectra; $V_L$ is the volume of the truncated radial model; $\gamma_L = \norm{\rho - \rho_0 - f_1 - \dots - f_L}_{L^2(d\mu)}$ is the truncation error of the radial function. The column ``components'' counts the powers and even bispectra.}
\label{tab-bennu-volume}
\end{table}

\begin{exa}[Bennu]\label[exa]{exa-bennu}
Let $c \in U_6$ be the tuple of Bennu above, let $T_{\mathrm{out}} (x) = \rho_0 + \max_{\Sph} s_x$ be the keep-out radius about the fixed reference origin, the center of mass of the radar model, and let $V (x)$ be the volume of the radial model $\rho_0 + s_x$. There are tuples $d_+, d_- \in U_6$ with coordinates in $10^{-20} \Z$ such that the following hold.
\begin{enumerate}
\item $\Phi_6 ( \tau c ) = \Phi_6 (c)$. The five odd bispectra of $c$ are nonzero, and each changes sign under $\tau$. Moreover $5 \leq \delta ( c, \tau c ) \leq 5.5138 \, \mathrm{m}$.
\item For $2 \leq L \leq 12$, the volume obtained from the even pool by \cref{thm-volume} equals the volume of the truncated model. In floating-point arithmetic the two agree to relative precision $2 \cdot 10^{-15}$, and $V_{\mathrm{pool}}$ differs from $V_{\mathrm{mesh}}$ by the amounts in \cref{tab-bennu-volume}. This part is numerical and is not certified.
\item For every component $\varphi$ of $\Phi_6$, $\abs{\varphi ( d_\pm ) - \varphi (c)} \leq 4 \cdot 10^{-19} \abs{\varphi (c)}$, and $\abs{V ( d_\pm ) - V (c)} \leq 2 \cdot 10^{-16} \, \mathrm{m}^3$.
\item The keep-out radii satisfy
\[
\begin{split}
272.161 & \leq T_{\mathrm{out}} ( d_- ) \leq 272.172 , \\
274.321 & \leq T_{\mathrm{out}} ( c ) \leq 274.332 , \\
279.588 & \leq T_{\mathrm{out}} ( d_+ ) \leq 279.599 ,
\end{split}
\]
in meters. In particular $d_+$ and $d_-$ lie neither in $\Or(3) \cdot c$ nor in each other's $\Or(3)$-orbit.
\item $4 \leq \delta ( c, d_+ ) \leq 4.2930 \, \mathrm{m}$.
\item $s_c$, $s_{d_+}$ and $s_{d_-}$ exceed $-40 \, \mathrm{m}$ on $\Sph$. Hence each of the radial functions $\rho_0 + s_c$, $\rho_0 + s_{d_+}$ and $\rho_0 + s_{d_-}$ exceeds $205 \, \mathrm{m}$ in every direction.
\end{enumerate}
\end{exa}

\noindent\textbf{Computation.}
(1) The values of $\Phi_6$ at $c$ and $\tau c$ are computed in exact arithmetic over $\Q (i)$ and coincide; this is also \cref{prop-mirror}(1). The odd bispectra $B_{234}$, $B_{245}$, $B_{256}$, $B_{346}$, $B_{456}$ are computed exactly at $c$ and $\tau c$. The distance bounds are obtained as in \cref{thm-decoy}(2), with $64\,176$ boxes, and the upper bound at a rational rotation listed in the data file.

(2) The identity \cref{eq-volume-expanded} is exact. In floating-point arithmetic both sides agree to the relative accuracy in \cref{tab-bennu-volume}. The last two columns of the table are numerical: they combine the truncation error of the radial model with the quadrature error of the projection, and neither is enclosed by a validated bound. \Cref{cor-volume-error} with $\rho_{\max} = 289.8 \, \mathrm{m}$ and $\gamma = \gamma_6$ bounds the volume error at $L = 6$ by $5.8 \, \%$ of $V_{\mathrm{mesh}}$; the actual error is $0.063 \, \%$.

(3) The tuples $d_\pm$ are found by numerical continuation along the level set of $\Phi_6$ through $c$, in the direction of the projection onto the quotient tangent space of the gradient of $\pm \bigl( \int ( \rho_0 + s_x )^{24} \, d\mu \bigr)^{1/24}$, followed by Newton iteration in $60$-digit arithmetic and rounding. The relative differences of the components are computed exactly, and are at most $3.7 \cdot 10^{-19}$. The volume bound is the componentwise estimate of \cref{cor-volume-error}, computed exactly.

(4) Each lower bound is the exact value of $\rho_0 + s_x$ at a rational point of $\Sph$ of the form $( 2u, 2v, \pm ( 1 - u^2 - v^2 ) ) / ( 1 + u^2 + v^2 )$ with $u, v \in \Q$. Each upper bound is established by the interval cover of \cref{thm-decoy}(5) applied to $- s_x$. The function $T_{\mathrm{out}}$ is invariant under $\Or(3)$, which gives the last assertion.

(5) As in \cref{thm-decoy}(2), with $43\,392$ boxes.

(6) As in \cref{thm-decoy}(5).

\begin{figure}[ht]
\centering
\includegraphics[width=\textwidth]{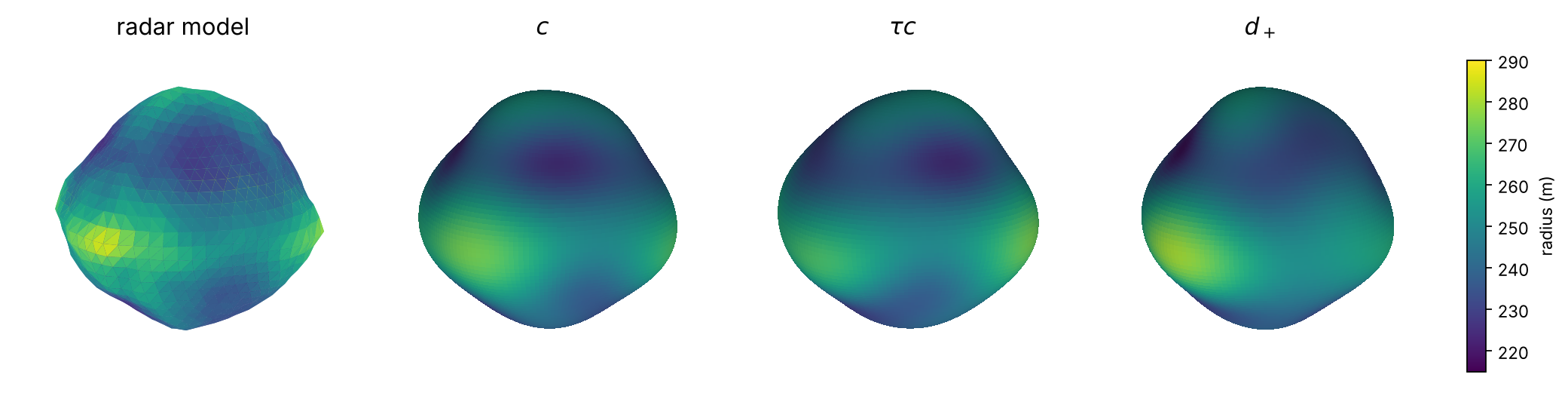}
\caption{Bennu: the radar shape model of \cite{NolanEtAl13}, its $L = 6$ model $c$, the mirror image $\tau c$, and the decoy $d_+$, labeled in each panel. The last two are rotated to their numerically best alignment with $c$. All four are drawn at the same scale. The color shows the distance from the origin in meters; see the color bar.}
\label{fig-bennu}
\end{figure}

\Cref{exa-bennu} has the following implications for the two adversarial perspectives and for the keep-out task. The even pool at $L = 6$ determines the modeled volume of Bennu, and hence its mass for a given bulk density. It does not determine the handedness of the asteroid: the mirror image matches all $32$ components exactly and lies at orbit distance at least $5 \, \mathrm{m}$. The odd bispectra do determine it. The even pool does not determine the keep-out radius either. The requirement $T_{\mathrm{out}} \leq 277 \, \mathrm{m}$ holds for $c$ and fails for $d_+$. The requirement $T_{\mathrm{out}} \leq 273 \, \mathrm{m}$ holds for $d_-$ and fails for $c$. For every admissible set $K$ containing $c$, $d_+$ and $d_-$, the task interval of $T_{\mathrm{out}}$ at the value $\Phi_6 (c)$ with componentwise relative tolerance $10^{-14}$ has width at least $279.588 - 272.172 = 7.416 \, \mathrm{m}$. By \cref{prop-task-information}, every estimate of the keep-out radius from that value has worst-case error at least $3.708 \, \mathrm{m}$ on these three tuples. The modeled volumes of $d_+$ and $d_-$ differ from that of $c$ by at most $2 \cdot 10^{-16} \, \mathrm{m}^3$. The continuation that produced $d_\pm$ imposes no centroid constraint. The example therefore audits the representation with the center-of-mass origin of the radar model held fixed, as in \cref{sec-admissibility}, and not a pipeline that recenters every candidate shape.

These statements concern the degree-six model of a radar shape model with a vertex spacing of about $25 \, \mathrm{m}$; see \cite{BennuPDS}. They do not describe the surface of Bennu at finer scales.

\section{Conclusion}\label{sec-conclusion}

The fiber radius measures geometric ambiguity, but does not by itself determine task utility or attribute confidentiality. Continuous ambiguity, discrete mirror ambiguity, finite-precision instability, and task ambiguity are distinct. The spectral pool gives a concrete example: its fibers can leave the shape undetermined while fixing enclosed volume exactly. The explicit clearance pair shows, without a numerical search, that a second geometric decision can remain impossible from the same descriptor. In an audit seeded by the radar shape model of Bennu, these phenomena appear together: the even pool fixes the modeled volume, misses the handedness, and, at componentwise relative descriptor tolerance $10^{-14}$, leaves the keep-out radius uncertain by more than $7 \, \mathrm{m}$. Together with task intervals, observer-centered reconstruction bounds, and the analysis of combined releases, this gives a precise basis for evaluating what a retained geometric representation can certify and what sharing it discloses.

\bibliography{ref}

\end{document}